\def\year{2021}\relax
\documentclass[letterpaper]{article}
\usepackage{aaai21arxiv} 
\usepackage{times} 
\usepackage{helvet} 
\usepackage{courier} 
\usepackage[hyphens]{url} 
\usepackage{graphicx} 
\usepackage{graphicx} 
\usepackage{natbib}
\usepackage{caption}
\usepackage{minitoc}
\usepackage{wrapfig}

\usepackage{microtype}
\usepackage{graphicx}
\usepackage{epstopdf}
\usepackage{booktabs} 
\usepackage[noline,ruled]{algorithm2e}

\usepackage{tikz}
\usepackage{standalone}
\usepackage{mathtools}
\usepackage{pgf}
\usepackage{subcaption}
\usetikzlibrary{shapes,positioning}
\usetikzlibrary{arrows.meta}
\usetikzlibrary{matrix}
\usetikzlibrary{calc}

\usetikzlibrary{matrix}
\usepackage{tkz-graph}
\usepackage{url}
\usepackage{color}\usepackage{xcolor} 
\colorlet{myorange}{red!30!yellow}
\usepackage{mathtools}
\usepackage{amssymb}
\usepackage{bbm}
\usepackage{etoolbox}
\usepackage{amsmath}
\usepackage{amsthm}
\usepackage{comment}
\usepackage{amsbsy}
\usepackage{amssymb}
\usepackage{graphicx}
\usepackage[normalem]{ulem}

\usepackage{lscape}
\usepackage{caption}
\usepackage{subcaption}

\newcommand{\michel}[1]{\textcolor{red}{{\itshape Michel:} #1}}
\renewcommand{\michel}[1]{}
\newcommand{\michelb}[1]{\textcolor{blue}{{\itshape Michel:} #1}}
\renewcommand{\michelb}[1]{}

\newtheorem{defn}{Definition}
\newtheorem{prop}{Proposition}
\newtheorem{corol}{Corollary}
\newtheorem{expl}{Example}
\newtheorem{model}{Model}

\makeatletter
\newcommand{\ostar}{\mathbin{\mathpalette\make@circled\star}}
\newcommand{\make@circled}[2]{%
	\ooalign{$\m@th#1\smallbigcirc{#1}$\cr\hidewidth$\m@th#1#2$\hidewidth\cr}%
}
\newcommand{\smallbigcirc}[1]{%
	\vcenter{\hbox{\scalebox{0.77778}{$\m@th#1\bigcirc$}}}%
}

\makeatother

\title{A Theory of Independent Mechanisms for Extrapolation in Generative Models$^*$}

\author {
	Michel Besserve,\textsuperscript{\rm 1,2}
	R\'emy Sun,\textsuperscript{\rm 1,3,$\dagger$}
	Dominik Janzing\textsuperscript{\rm 1,$\dagger$}
	and
	Bernhard Sch\"olkopf\textsuperscript{\rm 1} \\
}
\affiliations {
	\textsuperscript{\rm 1} Max Planck Institute for Intelligent Systems, T\"ubingen, Germany \\
	\textsuperscript{\rm 2} Max Planck Institute for Biological Cybernetics, T\"ubingen, Germany \\
	\textsuperscript{\rm 3} ENS Rennes, France \\
	\{michel.besserve, bs\}@tuebingen.mpg.de, janzind@amazon.de, remy.sun@ens-rennes.fr
}

\begin{document}


\maketitle
\begin{abstract}
Generative models can be trained to emulate complex empirical data, but are they useful to make predictions in the context of previously unobserved environments? An intuitive idea to promote such \textit{extrapolation} capabilities is to have the architecture of such model reflect a causal graph of the true data generating process, such that one can intervene on each node independently of the others. However, the nodes of this graph are usually unobserved, leading to overparameterization and lack of identifiability of the causal structure. We develop a theoretical framework to address this challenging situation by defining a weaker form of identifiability, based on the principle of \textit{independence of mechanisms}. We demonstrate on toy examples that classical stochastic gradient descent can hinder the model's extrapolation capabilities, suggesting independence of mechanisms should be enforced explicitly during training. Experiments on deep generative models trained on real world data support these insights and illustrate how the extrapolation capabilities of such models can be leveraged.
\end{abstract}
\renewcommand*{\thefootnote}{}
	\footnotetext{\hspace*{-.55cm}\textsuperscript{$\dagger\,\,$}DJ contributed before joining Amazon. RS contributed before joining Thales and Sorbonne University.}

\setcounter{footnote}{3}
\renewcommand*{\thefootnote}{\arabic{footnote}}
\section{Introduction}

Deep generative models such as Generative Adversarial Networks (GANs) \citep{goodfellow2014generative}, and
Variational Autoencoders (VAEs) \citep{kingma2013auto,rezende2014stochastic} are able to learn complex structured data such as natural images. However, once such a network has been trained on a particular dataset, can it be leveraged to simulate meaningful changes in the data generating process? 
Capturing the causal structure of this process allows the different mechanisms involved in generating the data to be intervened on independently, based on the principle of \textit{Independence of Mechanisms} (IM) \citep{janzing2010causal,LemeireJ2012,causality_book}.  IM reflects a foundational aspect of causality, related to concepts in several fields, such as superexogeneity in economics \citep{engle1983exogeneity}, the general concept of invariance in philosophy \citep{woodward} and \textit{modularity}. In particular, having the internal computations performed by a multi-layer generative model reflect the true causal structure of the data generating mechanism would thus endow it with a form of layer modularity, such that intervening on intermediate layers causes changes in the output distribution similar to what would happen in the real world. We call such ability \textit{extrapolation}, as it intuitively involves generalizing beyond the support of the distribution sampled during training, or its convex hull.


In this paper, \textbf{we focus on the challenging case where no additional variables, besides the samples from the data to generate, are observed} (in contrast with related work, as explained below). In this unsupervised setting, generative models are typically designed by applying successive transformations to latent variables, leading to a multi-layered architecture, where neither the latent inputs nor the hidden layers correspond to observed variables. We elaborate a general framework to assess extrapolation capabilities when intervening on hidden layer parameters with transformations belonging to a given group $\mathcal{G}$, leading to the notion of $\mathcal{G}$-genericity of the chosen parameters. We then show how learning based on stochastic gradient descent can hinder $\mathcal{G}$-genericity, suggesting additional control on the learning algorithm or the architecture is needed to enforce extrapolation abilities. Although we see our contribution as chiefly conceptual and theoretical, we use toy models and deep generative models trained on real world data to illustrate our framework.
\subsubsection{Appendix.} Readers can refer to the technical appendix in the extended version of this paper\footnote{\url{https://arxiv.org/abs/2004.00184}} for supplemental figures, code resources,  symbols and acronyms (Table~1), all proofs (App.~A) 
  and method details (App.~B).
\michel{check all appendices, proofs, symbols,... are refered correctly in main text}
\subsubsection{Related Work.}

Deep neural network have been leveraged in causal inference for learning causal graphs between observed variables \citep{lopez2016revisiting} and associated causal effects \citep{louizos2017causal,shalit2017estimating,kocaoglu2017causalgan,lachapelle2019gradient,zhu2019causal}. Our ultimate goal is more akin to the use of a causal framework to enforce domain adaptation \citep{zhang2013domain,zhang2015multi} and domain shift robustness of leaning algorithms, which has been done by exploiting additional information in the context of classification \citep{heinze2017conditional}. Broadly construed, this also relates to zero-shot learning \citep{lampert2009learning} and notions of extrapolations explored in the context of dynamical systems \citep{martius2016extrapolation}. As an intermediary step, unsupervised disentangled generative models are considered as a way to design data augmentation techniques that can probe and enforce the robustness of downstream classification tasks \citep{locatello2018challenging,higgins2016beta}. A causal (counterfactual) framework for such disentanglement has been proposed by \citet{besserve2018counterfactuals} that leverages the internal causal structure of generative models to generate meaningful changes in their output. In order to characterize and enforce such causal disentanglement properties, the IM principle has been exploited in empirical studies \citep{goyal2019recurrent,pmlr-v80-parascandolo18a} and 
its superiority to statistical independence has been emphasized \citep{besserve2018counterfactuals,locatello2018challenging}.  However, deriving a measure for IM is challenging in practice. Our work builds on the idea of \citet{besserve2018aistats} to use group invariance to quantify IM in a flexible setting and relate it to identifiability of the model in the absence of information regarding variables causing the observations. Another interesting direction to address identifiability of deep generative model is non-linear ICA, but typically requires observation of auxiliary variables \citep{hyvarinen2019nonlinear,khemakhem2020variational}. Finally, our investigation of  overparameterization relates to previous studies \citep{neyshabur2017geometry,zhang2016understanding}, notably arguing that Stochastic Gradient Descent (SGD) implements an \textit{implicit regularization} beneficial to supervised learning, while we provide a different perspective in the context of unsupervised learning and extrapolation. 

%

\section{Extrapolation in Generative Models}\label{sec:extrapol}
\subsection{\textit{FluoHair}: an Extrapolation Example in VAEs}
\michel{roadmap here or above}
We first illustrate what we mean by extrapolation, and its relevance to generalization and  generative models with a straightforward transformation: color change. ``Fluorescent'' hair colors are at least very infrequent in classical face datasets such as CelebA\footnote{\tiny\url{http://mmlab.ie.cuhk.edu.hk/projects/CelebA.html}}, such that classification algorithms trained on these datasets may fail to extract the relevant information from pictures of actual people with such hair, as they are arguably outliers. 

To foster the ability to generalize to such samples, one can consider using generative models to perform data augmentation. However, highly realistic generative models also require training on similar datasets, and are thus very unlikely to generate enough samples with atypical hair attributes. 

Fig.~\ref{fig:fluohair} demonstrates a way to endow a generative model with such extrapolation capabilities: after identifying channels controlling hair properties in the last hidden layer of a trained VAE (based on the approach of \citet{besserve2018counterfactuals}), the convolution kernel $k$ of this last layer can be modified to generate faces with various types of fluorescence (see App.~B.1  
  for details), while the shape of the hair cut, controlled by parameters in the above layers, remains the same, illustrating layer-wise modularity of the network. Notably, this approach to extrapolation is unsupervised: no labeling or preselection of training samples is used. Importantly, in our framework hair color is not controlled by a disentangled latent variable; we rely instead on the structure of VAE/GAN to intervene on color by changing the synaptic weights  corresponding to hidden units influencing hair in the last (downstream) convolution layer thereby influencing output RGB channels (see in App.~B.1). 
  Such transformation of an element of the computational graph of the generative model will guide our framework. Although this example provides insights on how extrapolations are performed, it exploits some features specific to color encoding of images. To illustrate how our framework helps address more general cases, we will use a different class of interventions that stretch the visual features encoded across a hierarchy of convolutional layers (Model~\ref{model:LTLCNN}, Fig.~\ref{fig:scaleEyeIllust}).

\subsection{Neural Networks as Structural Causal Models}
By selecting the output of a particular hidden layer as intermediate variable $\boldsymbol{V}$, we represent (without loss of generality) a multi-layer generative model as a composition of two functions $f_k^{\boldsymbol{\theta}_k}(.;\,)$, $k=1,2$, parameterized by $\boldsymbol{\theta}_k\in \mathcal{T}_k$, and applied successively to a latent variable $\boldsymbol{Z}$ with a fixed distribution, to generate an output random variable
\begin{equation}\label{eq:genPair}
	\boldsymbol{X} = f_{2}^{\boldsymbol{\theta}_{2}}(\boldsymbol{V}) = f_{2}^{\boldsymbol{\theta}_{2}}(f_{1}^{\boldsymbol{\theta}_1}(\boldsymbol{Z}))\,.
\end{equation}
Assuming the mappings $\boldsymbol{\theta}_k\mapsto f_k$ are one-to-one, we abusively denote parameter values by their corresponding function pair. Besides pathological cases, e.g. ``dead'' neurons resulting from bad training initialization, this assumption appears reasonable in practice.\footnote{the opposite would mean e.g. for a convolutional layer, that two different choices of tensors weights lead to the exact same response for all possible inputs, which appears unlikely} 

\begin{figure}[t]
		\includegraphics[width=\linewidth]{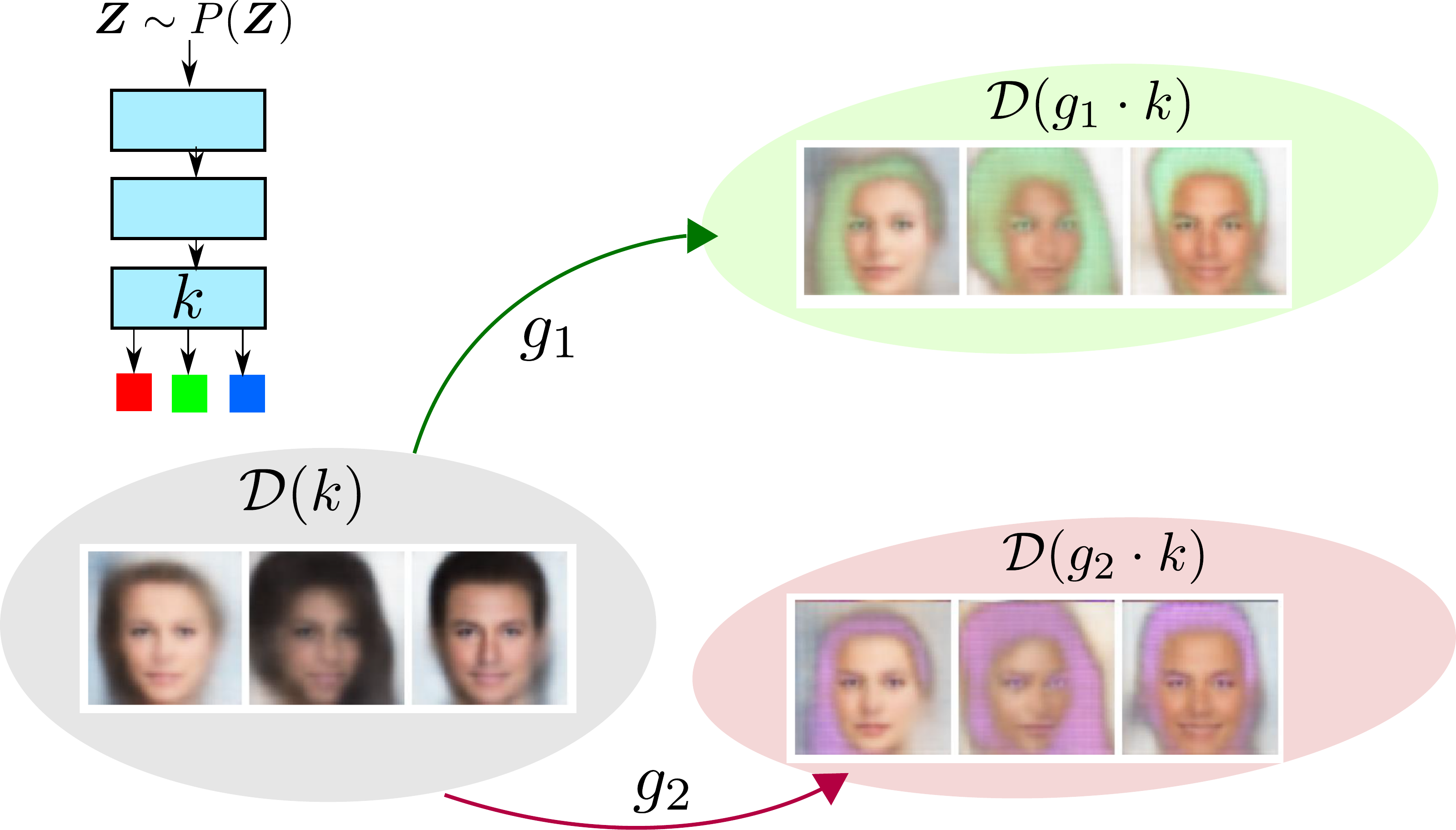}
		\caption{Illustration of \textit{FluoHair} extrapolation for a VAE face generator (left inset). Transformations $g_k$ modify kernel $k$ and the sample distribution $\mathcal{D}$.\label{fig:fluohair}}
\end{figure}
\begin{figure}[t]
		\includegraphics[width=\linewidth]{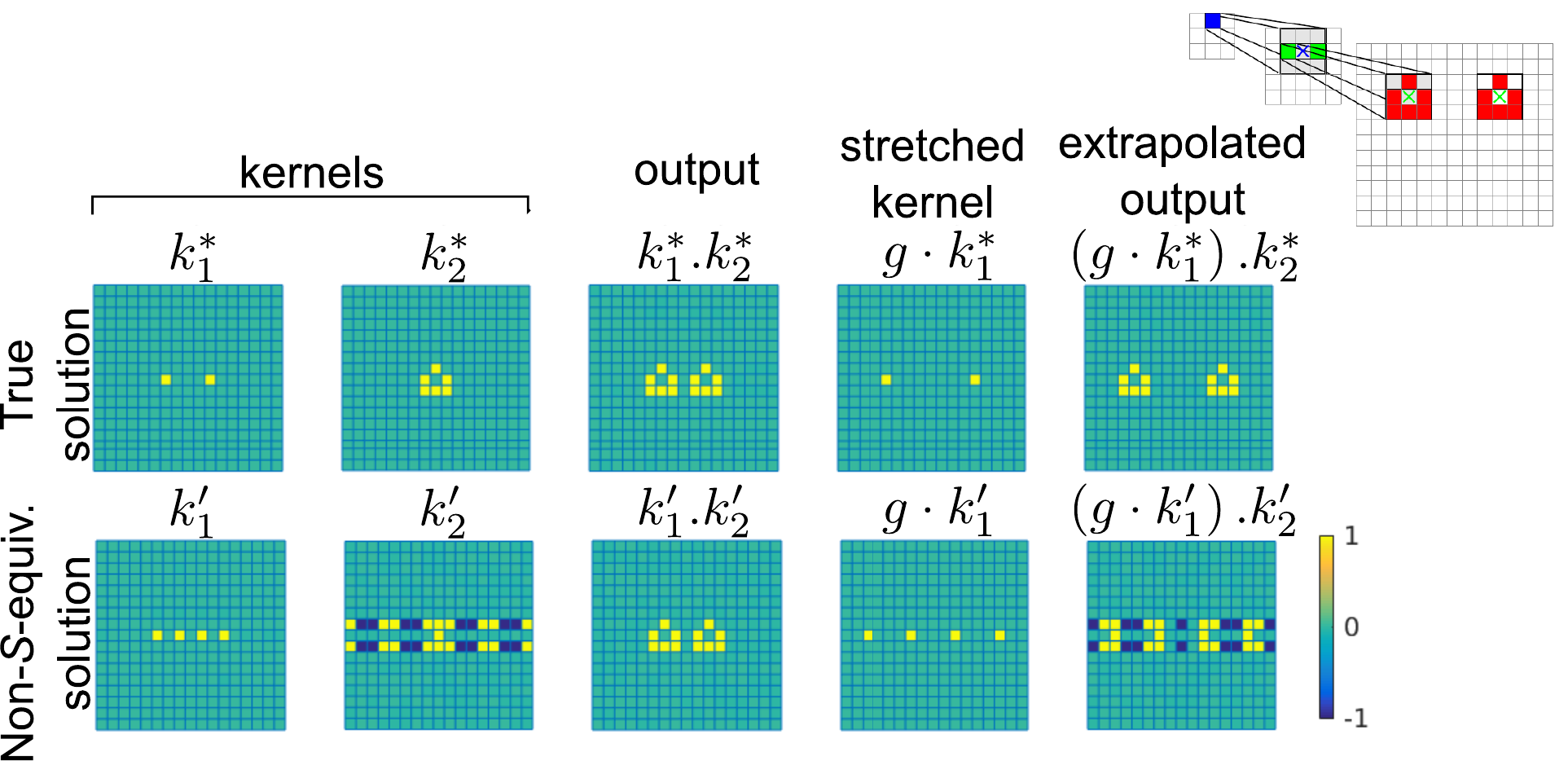}
		\caption{ Numerical illustration for eye generator of Example~\ref{expl:eyegen}. Top row: true parameters. Bottom: another solution in $S_{\boldsymbol{\theta}^*}$. (Top right inset) Illustration of the two successive convolutions (with additional striding).\label{fig:scaleEyeIllust}}
\end{figure}
An assumption central to our work is that the data generating mechanism leading to the random variable $\boldsymbol{Y}$ representing observations corresponds to eq.~(\ref{eq:genPair}) with the so-called \textit{true} parameters $\boldsymbol{\theta}^*\in\mathcal{T}$ corresponding to $(f_1^*,f_2^*)$. More precisely both functions $f_1^*$ and $f_2^*$ are assumed to capture causal mechanisms such that one can interpret eq.~(\ref{eq:genPair}) as a structural causal model \citep{pearl2000causality} with causal graph $\boldsymbol{Z}\rightarrow\boldsymbol{V}\rightarrow\boldsymbol{X}$.

We additionally assume that a learning algorithm fits perfectly the data distribution by choosing the vector of parameters $\widetilde{\boldsymbol{\theta}}$. This assumption allows us to focus on the theoretical underpinnings of extrapolation independent from the widely addressed question of fitting complex generative models to observations.  
In practical settings, this can be approached by choosing an architecture with universal approximation capabilities. Let $\mathcal{D}_{\boldsymbol{\theta}}$ denote the distribution of output $\boldsymbol{X}$ for any parameter pair $\boldsymbol{\theta}$ in $\boldsymbol{\mathcal{T}}= \mathcal{T}_1 \times \mathcal{T}_2$, then we have $\boldsymbol{Y}\sim \mathcal{D}_{\boldsymbol{\theta}^*}= \mathcal{D}_{\widetilde{\boldsymbol{\theta}}}$. 
The fitted parameters will thus belong to a \textit{solution set} $S_{\boldsymbol{\theta}^*}$, defined as a set of function pairs that fit the observational distribution perfectly: 
\begin{equation}\label{eq:solSet}
	S_{\boldsymbol{\theta}^*}=\{(f_1,f_2)| \mathcal{D}_{(f_1,f_2)}=\mathcal{D}_{\boldsymbol{\theta}^*}\}\,.
\end{equation}
If $\widetilde{\boldsymbol{\theta}}=\boldsymbol{\theta}^*$, we can predict the distribution resulting from interventions on these parameters in the real world.
We call such case \textit{structural identifiability}. The IM principle  at the heart of causal reasoning then allows \textit{extrapolation} to other plausible distributions of output $\boldsymbol{Y}$ by intervening on one function while the other is kept fixed (see \textit{FluoHair} example above). In contrast, if $S_{\boldsymbol{\theta}^*}$ is non-singleton and a value $\widetilde{\boldsymbol{\theta}}\neq\boldsymbol{\theta}^*$ is chosen by the learning algorithm, extrapolation is, in general, not guaranteed to behave like the true solution. One source on non-identifiability is the possibility that the pushforward measure of $\boldsymbol{Z}$ by two different functions $f_2\circ f_1= f \neq f'=f_2'\circ f_1'$ belonging to the model class may both match $\mathcal{D}_{\boldsymbol{\theta}^*}$ perfectly. In contrast, we will call \textit{functionally identifiable} a true parameter $\boldsymbol{\theta}^*$ such that the composition $f=f_2\circ f_1$ is uniquely determined by $\mathcal{D}_{\boldsymbol{\theta}^*}$. However, even a functionally identifiable parameter may not be structurally identifiable if $f$ may by obtained by composing different pairs $(f_1, f_2)$ and $(f_1',f_2')$. This last case is the focus of our framework, and will be illustrated using the following model.
\begin{model}[Linear 2-layer convNet]\label{model:LTLCNN}
	Assume $d,d'$ are two prime numbers\footnote{This will allow defining rigorously a group of transformations for extrapolation.},  $\boldsymbol{Z}$ a $(2d-1)\times (2d'-1)$ random binary latent image, such that one single pixel is set to one at each realization, and probability of this pixel to be located at $(i,j)$ is $\boldsymbol{\pi}_{i,j}$. Let $(k_1,\,k_2)$ be two invertible $(2d-1)\times (2d'-1)$ convolution kernels, and 
	\begin{equation}\label{eq:circonv}
		\boldsymbol{X}=k_2\ostar \boldsymbol{V}=k_2\ostar k_1\ostar \boldsymbol{Z},
	\end{equation}
	where $\ostar$ is the circular convolution (modulo $2d-1,2d'-1$).
\end{model}
The reader can refer to App.~B.3 
 for a background on circular convolution and how it relates to convolutional layers in deep networks. 
Such model can be used to put several copies of the same object in a particular spatial configuration at a random position in an image. The following example (Fig.~\ref{fig:scaleEyeIllust}) is an ``eye generator'' putting an eye shape at two locations separated horizontally by a fixed distance in an image to model the eyes of a (toy) human face. The location of this ``eye pair'' in the whole image may also be random.
\begin{expl}[Eye generator, Fig.~\ref{fig:scaleEyeIllust}]\label{expl:eyegen}
	Consider Model~\ref{model:LTLCNN} with  $k_2$ a convolution kernel taking non-zero values within a minimal square of side $\delta<d$ encoding the eye shape, and $k_1$ with only two non-vanishing pixels, encoding the relative position of each eye.
\end{expl}

\subsection{Characterization of the Solution Set}\label{sec:solset}
In the context of training such model from data without putting explicit constraints on each kernel, Model~\ref{model:LTLCNN} admits ``trivial'' alternatives to the true parameters $(k_1^*,k_2^*)$ to fit the data perfectly, simply by left-composing arbitrary rescalings and translations with $k_1^*$, and right-composing the inverse transformation to $k_2^{*}$. This is in line with observations by \citet{neyshabur2017geometry} in ReLU networks (incoming synaptic weights of a hidden unit can be downscaled while upscaling all outgoing weights). 

To go beyond these mere observations, we systematically characterize over-parameterization entailed by composing two functions. 
Let $\boldsymbol{\mathcal{V}}$ 
be the range 
of $\boldsymbol{V}$, we define the subset $\Omega$ of right-invertible functions $\omega:\boldsymbol{\mathcal{V}}\rightarrow\boldsymbol{\mathcal{V}}$ such that for any pair $(f_1,\,f_2)$, $(\omega^{-1} \circ f_1, f_2 \circ \omega)$ also corresponds to a valid choice of model parameters.\footnote{We use the convention $A\circ \omega = \{f \circ \omega,\,f\in A\}$.} Trivially, $\Omega$ contains at least the identity map. For any true parameter $\boldsymbol{\theta}^*$, we define the \textit{Composed Over-parameterization Set} (COS)
\begin{equation}\label{eq:COS}
	S^\Omega_{\boldsymbol{\theta}^*}=\left\{\left(\omega^{-1}\circ f_1^{\boldsymbol{\theta}_1^*},\, f_2^{\boldsymbol{\theta}_2^*}\circ \omega\right)|\,\omega\in \Omega\right\}\,.
\end{equation}
The COS reflects how ``internal'' operations in $\Omega$ make the optimization problem under-determined because they can be compensated by internal operations in neighboring layers. By definition, the COS is obviously a subset of the solution set $S_{\boldsymbol{\theta}^*}$. But if we consider \textit{normalizing flow} (NF) models, in which $f_k$'s are always invertible (following \citet{rezende2015variational}), we can show inclusion turns into equality.
\begin{prop}\label{prop:COSequalS}
	For an NF model, $\Omega$ is a group and for any functionally identifiable true parameter $\boldsymbol{\theta}^*$, $S^\Omega_{\boldsymbol{\theta}^*}=S_{\boldsymbol{\theta}^*}$.
\end{prop}
Notably, this result directly applies to Model~\ref{model:LTLCNN} (see Corollary~1 
 in App.~B.6). 
We will exploit the COS group structure to study the link between identifiability and extrapolation, which we define next.

\subsection{Extrapolated Class of Distributions}\label{sec:metaclass}
\michel{improve second part of section}
Humans can generalize from observed data by envisioning objects that were not previously observed, akin to our \textit{FluoHair} example (Fig.~\ref{fig:fluohair}). To mathematically define the notion of extrapolation, we pick interventions from a group $\mathcal{G}$ (i.e. a set of composable invertible transformations, see App.~B.2 
for background) to manipulate the abstract/internal representation instantiated by vector $\boldsymbol{V}=f_1(\boldsymbol{Z})$. Given parameter pair $\boldsymbol{\theta}^*=(f_1^*,f_2^*)$,
we then define the $\mathcal{G}$-\textit{extrapolated model class}, which contains the distributions generated by the interventions on $\boldsymbol{V}$ through the action of the group on $f_1$: 
\begin{equation}\label{eq:metaclass}
	\mathcal{M}^\mathcal{G}_{(f_1^*,f_2^*)}=\mathcal{M}^\mathcal{G}_{\boldsymbol{\theta}^*}\triangleq \left\{\mathcal{D}_{(g\cdot f_1^*,\, f_2^*)},\,g\in \mathcal{G} \right\}
	\,,
\end{equation}
where $g\cdot f_1^*$ denotes the group action of $g$ on $f_1^*$, transforming it into another function. $\mathcal{G}$ thus encodes the inductive bias used to extrapolate from one learned model to others (when $\mathcal{G}$ is unambiguous $\mathcal{M}^\mathcal{G}_{\boldsymbol{\theta}^*}$ is denoted $\mathcal{M}_{\boldsymbol{\theta}^*}$).
\michel{note on different actions removed} An illustration of the principle of extrapolated class, is provided in Suppl.~ Fig.~2.

Choosing the set of considered interventions to have a group structure allows to have an unequivocal definition of a uniform (Haar) measure on this set for computing expectations and to derive interesting theoretical results. Note this does not cover non-invertible hard interventions that set a variable to a fixed value $y=y_0$, while  shifting the current value by a constant $y\rightarrow y+g$ does fit in the group framework. In the context of neural networks, this framework also allows to model a family of interventions on a hidden layer which can be restricted to only part of this layer, as done in the \textit{FluoHair} example (see App.~B.1). 

The choice of the group is a form of application-dependent inductive bias. For Model~\ref{model:LTLCNN}, a meaningful choice is the multiplicative group $\mathcal{S}$ of integers modulo $d$ (with $d$ prime number, see App.~B.5), 
such that the group action of stretching the horizontal image axis $\{-d+1,..,0,d-1\}$ by factor $g\in \mathcal{S}$ turns convolution kernel $k$ into $(g\cdot k)(m,n)=k(gm,gn)$. Such stretching is meant to approximate the rescaling of a continuous axis, while preserving group properties,
and models classical feature variations in naturalistic images (see  App.~B.5). 
 As an illustration for Example~\ref{expl:eyegen}, using this group action leads to an extrapolated class that comprises models with various distances between the eyes, corresponding to a likely variation of human face properties. See Fig.~\ref{fig:scaleEyeIllust}, top row for an example extrapolation using this group.  Interestingly, such spatial rescalings also correspond to frequency axis rescalings in the Fourier domain (see background in App.~B.4). 
  Indeed, let $\widehat{k}$ be the Discrete Fourier Transform (DFT) of kernel $k$, $(g\cdot k)(m,n)=k(gm,gn)$ corresponds to $(g\cdot \widehat{k})(u,v)=\widehat{k}(ug^{-1},ng^{-1})$ such that the frequency axis is rescaled by the inverse of $g$. Due to the relationship between convolution and Fourier transform (App.~B.4), 
  several results for Model~\ref{model:LTLCNN} will be expressed in the Fourier domain where convolution acts as a diagonal matrix multiplication.

\subsection{Extrapolation Replaces Identification: $\mathcal{G}$-equivalence and $\mathcal{G}$-genericity}
\michel{Merge and simplify this section and the next}
As elaborated above, 
$
S_{\boldsymbol{\theta}^*}$
may not be singleton such that a solution $(\tilde{f}_1,\,\tilde{f}_2)$ found by the learning algorithm may not be the true pair $(f_1^*,\,f_2^*)$, leading to a possibly different extrapolated class when intervening on $f_1$ with elements from group $\mathcal{G}$. When extrapolated classes happen to be the same, we say the solution is $\mathcal{G}$-\textit{equivalent} to the true one.
\begin{defn}[$\mathcal{G}$-equivalence]
	The solution $(\tilde{f}_1,\,\tilde{f}_2)$ is $\mathcal{G}$-\textit{equivalent} to the true $(f_1^*,\,f_2^*)$ if it generates the same extrapolated class through the action of $\mathcal{G}$: ${\mathcal{M}^{\mathcal{G}}_{(\tilde{f}_1,\,\tilde{f}_2)}}=\mathcal{M}_{(f_1^*,f_2^*)}^{\mathcal{G}}$.	
\end{defn}
An illustration of $\mathcal{G}$-equivalence violation for Example~\ref{expl:eyegen} is shown in Fig.~\ref{fig:scaleEyeIllust} (bottom row), and an additional representation of the phenomenon is given in Suppl.~Fig.~2. 
Such equivalence of extrapolations imposes additional requirements on solutions. In the NF cases, such constraints rely on the interplay between the group structure of $\Omega$ (group of the COS in eq.~(\ref{eq:COS})) that constrains over-parameterization, and the group structure of $\mathcal{G}$. For Model~\ref{model:LTLCNN}, in the 1D case this leads to
\michel{removed detailed explanation, add again if possible} 
\begin{prop}\label{prop:convFaith}
	Assume $\widehat{\boldsymbol{\pi}}$ has no zero element and d'=1, the solution $({k}_1,\,{k}_2)$ for Model~\ref{model:LTLCNN} is $\mathcal{S}$-equivalent to true model $(k_1^*,\,k_2^*)$ if and only if there exists one $\lambda\in \mathbb{C}$ such that $(\widehat{k}_1(u)],\,\widehat{k}_2(u)])=(\lambda^{-1}\widehat{k}_1^*(u),\, \lambda\widehat{k}_2^*(u))$ for all $u>0$.
\end{prop}
This shows that at least in this model, $\mathcal{G}$-equivalence is achieved only for solutions that are very similar to the true parameters $\mathcal{\theta}^*$ (up to a multiplicative factor), thus only slightly weaker than identifiability.
As $\mathcal{G}$-equivalence requires knowledge of the true solution, in practice we resort to characterizing invariant properties of $\mathcal{M}_{\boldsymbol{\theta}^*}$ to select solutions. Indeed, if $\mathcal{M}_{\boldsymbol{\theta}^*}$ is a set that ``generalizes'' the true model distribution $\mathcal{D}_{\boldsymbol{\theta}^*}$, it should be possible to express the fact that some property of $\mathcal{D}_{\boldsymbol{\theta}^*}$ is \textit{generic} in $\mathcal{M}_{\boldsymbol{\theta}^*}$. Let ${\varphi}$ be a \textit{contrast} function capturing approximately the relevant property of $\mathcal{D}_{\boldsymbol{\theta}^*}$, we check that such function does not change on average when applying random transformations from $\mathcal{G}$, by sampling from the Haar measure of the group $\mu_\mathcal{G}$,\footnote{$\mu_\mathcal{G}$ is a ``uniform'' distribution on $\mathcal{G}$, see App.~B.2} 
leading to  \michelb{improve}
\begin{defn}[Contrast based $\mathcal{G}$-genericity]\label{def:gDisent}
	Let ${\varphi}$ be a function mapping distributions on the generator output space to $\mathbb{R}$, and $\mathcal{G}$ a compact group. For any solution $(\tilde{f}_1,\,\tilde{f}_2)$ of the model fit procedure, we define the \textit{generic ratio}
	\begin{equation}\label{eq:genRatio}
		\rho(\tilde{f}_1,\,\tilde{f}_2)=\rho(\tilde{f}_1(\boldsymbol{Z}),\,\tilde{f}_2)\triangleq\frac{{\varphi}(\mathcal{D}_{(\tilde{f}_1,\,\tilde{f}_2)})}{\mathbb{E}_{g\sim\mu_\mathcal{G}}{\varphi}(\mathcal{D}_{(g\cdot\tilde{f}_1,\,\tilde{f}_2)})}\,.
	\end{equation}
	Solution $(\tilde{f}_1,\,\tilde{f}_2)$ is $\mathcal{G}$-generic w.r.t. ${\varphi}$, whenever it satisfies $\rho(\tilde{f}_1,\,\tilde{f}_2) = 1$.
\end{defn}
It then follows trivially from the definition that $\mathcal{G}$-equivalence entails a form of $\mathcal{G}$-genericity.
\begin{prop}\label{disentFaith}
	For ${\varphi}$ constant on $\mathcal{M}_{\boldsymbol{\theta}^*}^{\mathcal{G}}$, $\mathcal{G}$-equivalent to the true solution implies $\mathcal{G}$-generic w.r.t. ${\varphi}$.
\end{prop}
Genericity was originally defined by \citet{besserve2018aistats} as a measure of \textit{independence} between cause $\boldsymbol{V}=f_1(\boldsymbol{Z})$ and mechanism $f_2$. In practice, genericity is not expected to hold rigorously but approximately (i.e. $\rho$ should be close to one). In the remainder of the paper, we use interchangeably the ``functional'' notation $\rho(\tilde{f}_1,\,\tilde{f}_2)$ and the original ``cause-mechanism'' notation $\rho(\boldsymbol{V},\,\tilde{f}_2)$.

\michel{possibly remove}

\subsection{Link Between Genericity and Direction of Causation}\label{sec:anticausal}
An interesting application of genericity is identifying the direction of causation : in several settings, if $\varphi(\tilde{f}_1(\boldsymbol{Z}),\,\tilde{f}_2)=1$ for the causal direction $\boldsymbol{V}\rightarrow \boldsymbol{X}$, reflecting a genericity assumption in a causal relation, then the \textit{anti-causal} direction $\boldsymbol{X}\rightarrow \boldsymbol{V}$ is not generic as  $\varphi(\boldsymbol{X},\,\tilde{f}_2^{-1})\neq 1$. 
As genericity, as measured empirically by its ratio, is only approximate (ratio not exactly equal to one), comparing genericity of the directions $\boldsymbol{Z}\rightarrow \boldsymbol{X}$ and $\boldsymbol{X}\rightarrow \boldsymbol{Z}$ can be used to support the validity of the genericity assumption. This comparison is supported by several works on identifiation of causal pairs using IM, showing identifiability can be obtained on toy examples by choosing the direction of causation that maximized genericity \citep{shajarisale2015,frW_UAI,icml2010_062,janzing2012information}. We use spectral independence to check genericity of neural network architectures in Sec.~\ref{sec:expDeep}.

\subsection{Scale and Spectral Independence}\label{sec:scaleIM}
In the case of Example~\ref{expl:eyegen} and for stretching transformations, restricted to the 1D case (d'=1), one reasonable contrast is the total \textit{Power} across non-constant frequencies, which can be written (see App.~B.4) 
\begin{equation}
	\label{eq:pow}
	\mathcal{P}(\textbf{X})=\frac{1}{d-1}\sum_{i\neq 0} |\widehat{k}_2(i)\widehat{k}_1(i)|^2=\left\langle |\widehat{k}_2\odot\widehat{k}_1|^2\right\rangle\,,
\end{equation}
where $\langle .\rangle$ denotes averaging over non zero frequencies and $\odot$ is the entrywise product. 
Indeed, this quantity is preserved when we stretch the distance between both eyes, as long as they do not overlap. The following result allows to exploit genericity to find a good solution:
\begin{prop}[Informal, see App.~A]
	\label{prop:toySIC}
	For Model~\ref{model:LTLCNN} in the 1D case, the $\mathcal{S}$-generic ratio with respect to $\varphi=\mathcal{P}$ is
	\begin{equation}
		\label{eq:sdr}
		\rho{(\boldsymbol{V}, {k_2})}=\frac{\langle\mathbb{E} |\widehat{k}_2\odot \widehat{\boldsymbol{V}}|^2\rangle  }{\langle \mathbb{E}|  \widehat{\boldsymbol{V}}|^2\rangle \langle |\widehat{k}_2|^2 \rangle }=\frac{\langle |\widehat{k}_2\odot \widehat{k}_1|^2\rangle  }{\langle |\widehat{k}_1|^2\rangle \langle |\widehat{k}_2|^2 \rangle }\,,
	\end{equation}
	Moreover, the true solution of Example~\ref{expl:eyegen} is $\mathcal{S}$-generic. 
\end{prop}
We call $\rho$ the Spectral Density Ratio (SDR), as it appears as a discrete frequency version of the quantity
introduced by \citet{shajarisale2015} (baring the excluded zero frequency).
We say such $\mathcal{S}$-generic solution w.r.t. $\rho$ satisfies \textit{scale or spectral independence}.
This supports the use of SDR to check whether successive convolution layers implement mechanisms at independent scales.\michelb{transition}

\section{How Learning Algorithms Affect Extrapolation Capabilities}\label{sec:opttot}
\subsection{Simplified Diagonal Model}\label{sec:diaggene}
\michel{we first show on a linear model...}
When models are over-parameterized, the learning algorithm likely affects the choice of parameters, and thus the extrapolation properties introduced above. We will rely on a simplification of Model~\ref{model:LTLCNN}, that allows to study the mechanisms at play without the heavier formalism of convolution operations.
\begin{model}\label{model:ABmodel}
	Consider the linear generative model of dimension $d-1$ with $d$ prime number
	\begin{equation}\label{eq:genLin}
		\boldsymbol{X} = \mathbf{\rm A B} \boldsymbol{Z} = {\rm diag}(a){\rm diag}(b) \boldsymbol{Z} 
	\end{equation}
	with $A$, $B$ square positive definite $(d-1) \times (d-1)$ diagonal matrices with diagonal coefficient vectors $a$ and $b$, respectively, and $\boldsymbol{Z}$ a vector of positive independent random variables such that $\mathbb{E} |Z_k|^2=1,\forall k$. 
\end{model}
Model~\ref{model:ABmodel} can be seen as a Fourier domain version of Model~\ref{model:LTLCNN}, with some technicalities dropped. In particular, we use real positive numbers instead of complex numbers, we drop the zero and negative frequencies by labeling dimensions as $\{1,2,...,d-1\}$ modulo $d$ and considering the multiplicative action of $\mathcal{S}$ on these coordinates. We get analogous results as for Model~\ref{model:LTLCNN} regarding the solution set and $\mathcal{S}$-equivalence (see Corol.~2 
and Prop.~7 
in App.~B.6).

In order to measure genericity in a similar way as for Model~\ref{model:LTLCNN}, the power contrast becomes\footnote{This contrast is used for causal inference with the \textit{Trace Method}  \citep{icml2010_062}, and relates to spectral independence \citet{shajarisale2015}.}
\[
\tilde{\varphi}(B,A)=\tau\left[AB B^\top A^\top \right]=\frac{1}{d-1}\sum_{i=1}^d a_i^2 b_i^2=\left\langle a^2\odot b^2\right\rangle
\]
where $\tau[M]$ is the normalized trace $\frac{1}{d-1}\text{Tr}[M]$. This leads to 
\begin{prop}\label{prop:TR}
	In Model~\ref{model:ABmodel}, the $\mathcal{S}$-generic ratio w.r.t. $\tilde{\varphi}(B,A)$ is
	$\displaystyle \quad	\rho'(B,A)\triangleq \frac{\left\langle a^2\odot b^2\right\rangle}{\left\langle a^2\right\rangle\left\langle b^2 \right\rangle}\,.$
\end{prop}

\subsection{Drift of Over-parameterized Solutions}\label{sec:optToy}
Consider Model~\ref{model:ABmodel} in the (degenerate) case of $1\times 1$ matrices. To make the learning closer to a practical setting, we consider a VAE-like training: conditional on the latent variable $z=Z$, the observed data is assumed Gaussian with fixed variance $\sigma^2$ and mean given by the generator's output $a\cdot b\cdot z$ (e.g. in contrast to \ref{eq:genPair}, noise is added after applying the second function, and one can retrieve the original setting in the limit case $\sigma^2=0$). To simplify the theoretical analysis, we study only the decoder of the VAE, and thus assume a fixed latent value $z=1$, (i.e. the encoder part of the VAE infers a Dirac for the posterior of $Z$ given the data). Assuming the true model $(a^*>0,b^*>0)$, we thus use data sampled from $\mathcal{N}(c=a^*b^*,\sigma^2)$, and learn $(a,b)$ from it, assuming the data is sampled from a Gaussian with same variance and unknown mean parameter. First, considering infinite amounts of data, maximum likelihood estimation amounts to minimizing the KL-divergence between two univariate Gaussians with same variance and different mean, equivalent to:
\begin{equation}\label{eq:toyLoss}
	\underset{a,b>0}{\text{minimize}}\quad \mathcal{L}(c;(a,b))=|c-ab|^2\,.
\end{equation}
We study the behavior of deterministic continuous time gradient descent (CTGD) in Prop.~8 
of App.~B.7. 
Typical trajectories are represented in red on Fig.~\ref{fig:toyTrajectories}.
We then consider the practical setting of SGD (see App.~B.8) 
  for training the VAE's decoder on the stochastic objective  
\begin{equation}
	\underset{a,b>0}{\text{minimize}}\,\, \ell(c_0;\omega;(a,b))=|C(\omega)-ab|^2,\, C\sim\mathcal{N}(c_0,\sigma^2)\,.
\end{equation}
The result (green sample path Fig.~\ref{fig:toyTrajectories}) is very different from the deterministic case, as the trajectory drifts along $S_{c_0}$ to asymptotically reach a neighborhood of $(\sqrt{c_0},\sqrt{c_0})$. This drift is likely caused by asymmetries of the optimization landscape in the neighborhood of the optimal set $S_{c_0}$. This phenomenon relates to observations of an implicit regularization behavior of SGD \citep{zhang2016understanding, neyshabur2017geometry}, as it exhibits the same convergence to the minimum Euclidean norm solution. We provide a mathematical characterization of the drift in Prop.~9 
(App.~B.8). 
This result states that an SGD iteration makes points in the neighborhood of $S_{c_0}$ evolve (on average) towards the line $\{a=b\}$, such that after many iterations the distribution concentrates around $(\sqrt{c_0},\sqrt{c_0})$.
Interestingly, if we try other variants of stochastic optimization on the same deterministic objective, we can get different dynamics for the drift, suggesting that it is influenced by the precise algorithm used (see App.~B.9 
  for the case of Asynchronous SGD (ASGD) and example drift in blue on Fig.~\ref{fig:toyTrajectories}).
\begin{figure}
	\begin{minipage}[t]{.47\linewidth}		
		\includegraphics[width=\linewidth]{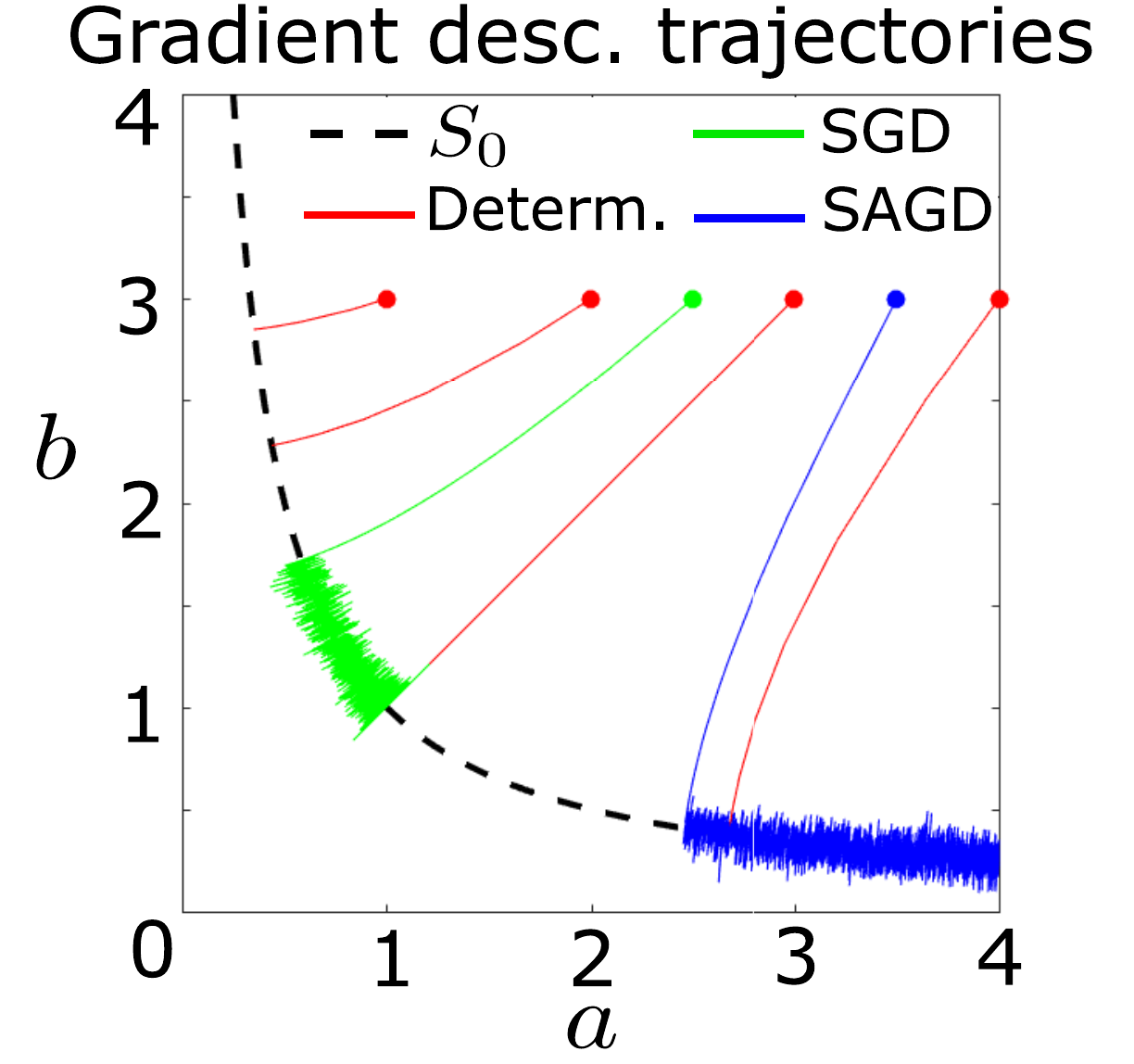}
		\subcaption{\label{fig:toyTrajectories}}
	\end{minipage}\hfill
	\begin{minipage}[t]{.53\linewidth}
		\includegraphics[width=\linewidth]{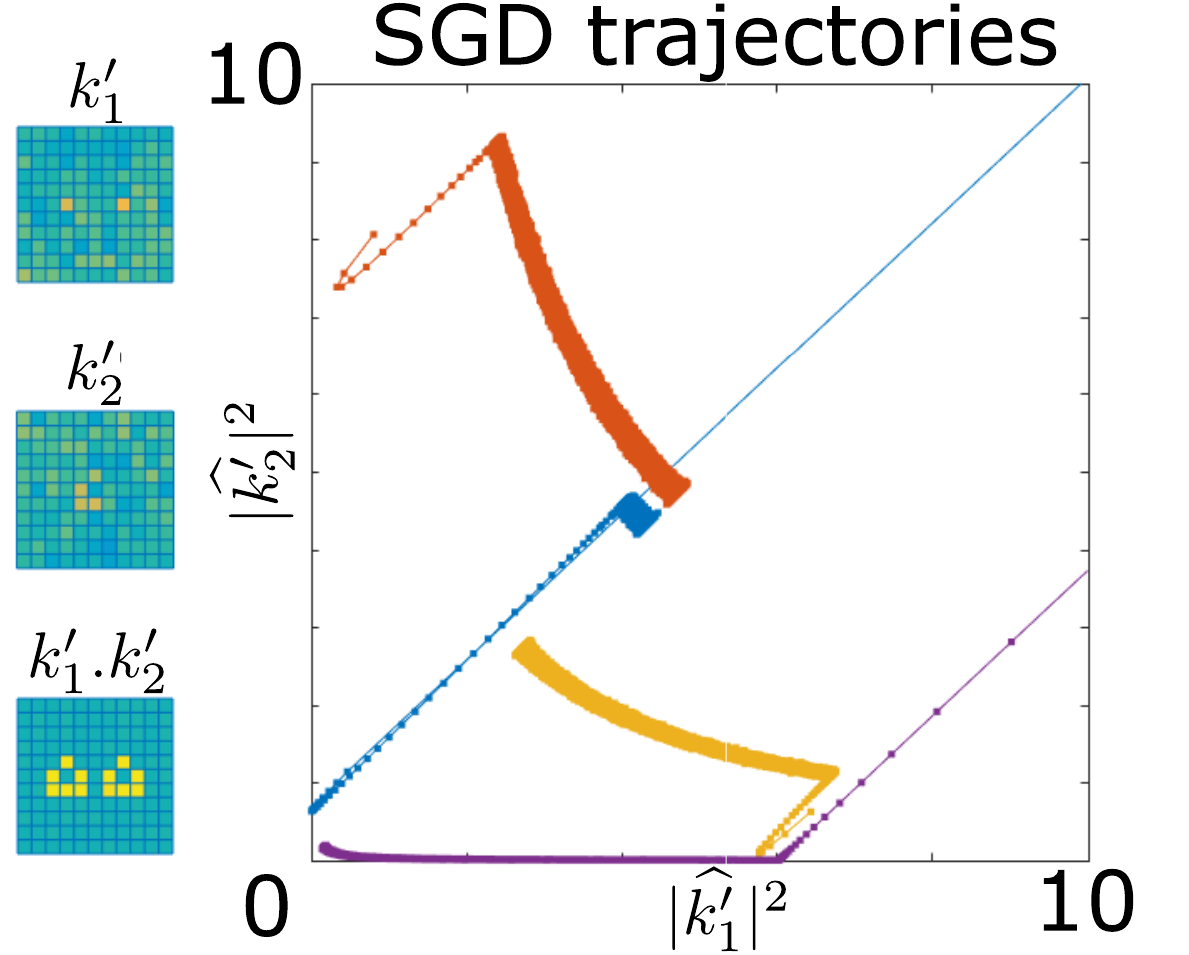}
		\subcaption{\label{fig:eyeTrajectories}}
	\end{minipage}
	\caption{(a) Gradient descent trajectories on the toy example of equation~(\ref{eq:toyLoss}), $c=1$. Thick dots indicate initial value. (b) SGD trajectories of several Fourier coefficients for Example~\ref{expl:eyegen}. Final kernels obtained are on left.}
\end{figure}

We now get back to the multidimensional setting for Model~\ref{model:ABmodel}. 
The above SGD results trivially apply to each component, which evolve independently from each other. Importantly, the next proposition shows that the SGD solution then drifts towards the matrix square root solution $\sqrt{A^*B^*}$ for both factors, leading to a violation of genericity. 
\begin{prop}\label{prop:TRlim}
	In Model~\ref{model:ABmodel}, assume diagonal coefficients of the true parameters $A^*$ and $B^*$ are i.i.d. sampled from two arbitrary non constant distributions. Then, the approximation of SGD solution $A=B=\sqrt{A^*B^*}$ satisfies 
	\[
	\rho'(B,\,A)\!\underset{d\rightarrow+\infty}{\longrightarrow}\! \mathbb{E}[c_1^2]/\mathbb{E}[c_1]^2> 1, \mbox{and is thus not $\mathcal{S}$-generic.
	}
	\]
\end{prop}
The solution chosen within $S_c$ by the SGD algorithm is thus suboptimal for extrapolation.

\subsection{Extension to Convolutional Model~\ref{model:LTLCNN}}
We show qualitatively how the above observations for Model~\ref{model:ABmodel} can provide insights for Model~\ref{model:LTLCNN}. Using the same VAE-like SGD optimization framework for this case, where we consider $\boldsymbol{Z}$ deterministic, being this time a Dirac pixel at location $(0,0)$. We apply the DFT to $\boldsymbol{X}$ in Model~\ref{model:LTLCNN} and use the Parseval formula to convert the least square optimization problem to the Fourier domain (see App.~B.4). 
 Simulating SGD of the real and imaginary parts of $\widehat{k}_1$ and $\widehat{k}_2$, we see in Fig.~\ref{fig:eyeTrajectories} the same drift behavior towards solutions having identical squared modulus ($|\widehat{k}_1|^2=|\widehat{k}_2|^2$), as described for Model~\ref{model:ABmodel} in Sec.~\ref{sec:optToy}, reflecting the violation of $\mathcal{S}$-genericity by SGD of Prop.~\ref{prop:TRlim}. As 
\begin{equation}\label{eq:trFreq}
	\rho'(|\widehat{k}_1^*|^2,|\widehat{k}_2^*|^2)= \rho(k_1^*,k_2^*)\,.
\end{equation}
this supports a violation of $\mathcal{S}$-genericity for the convolution kernels, such that the SGD optimization of Model~\ref{model:LTLCNN} is also suboptimal for extrapolation.
\subsection{Enforcing Spectral Independence}
\label{sec:opt}
In order to enforce genericity and counteract the effects of SGD, we propose to alternate the optimization of the model parameters with SDR-based genericity maximization. To achieve this, we multiply the square difference between the SDR and its ideal value of 1 by the normalization term $\left\langle|\widehat{k}_2^i|^2\right\rangle$ and alternate SGD steps of the original objective with gradient descent steps of the following problem 
\begin{equation}
	\label{eq:proxy}
	\textstyle
	\underset{\widehat{k}_2}{\rm minimize} 
	\left(\rho(k_1^*,k_2^*)-1\right)^2\langle |  \widehat{k_2}|^2\rangle^2
	=\left\langle \!\!
	|{\widehat{{k}_2}}|^2 \!\odot\!\left(\!\!\frac{ | \widehat{k_1}|^2  }{\langle|\widehat{k_1}|^2\rangle}\!\! -1 \! \right) \!\!\right\rangle^2\!\!.
\end{equation}
Performance of this procedure is investigated in App.~B.12.


\section{Experiments on Deep Face Generators}\label{sec:expDeep}
We empirically assess extrapolation abilities of deep convolutional generative
networks, in the context of learning the distribution of CelebA. We used a plain
$\beta$-VAE\footnote{\tiny\url{https://github.com/yzwxx/vae-celebA}} (\cite{higgins2016beta}) and the official tensorlayer DCGAN
implementation\footnote{\tiny\url{https://github.com/tensorlayer/dcgan}}. The general structure of the VAE is summarized in Suppl. 
Fig.~1b 
and the DCGAN architecture is very similar (details in~ Suppl. Table~2). 
Unless otherwise stated, our analysis is done on the generative architecture of these models (VAE encoder, GAN generator). We
denote the 4 different convolutional layers as indicated in Suppl. Fig.~1b: 
coarse (closest to latent variables), intermediate, fine and image level. The theory developed in previous sections was adapted to match these applied cases, as explained in App.~B.10. 

\subsection{Stretching Extrapolations}
Extrapolations were performed by applying a 1.5 fold horizontal stretching transformation to all maps of a given hidden convolutional layer and compare the resulting perturbed image to directly stretching to the output sample. 
\begin{figure*}
	\centering	\includegraphics[width= .9\linewidth]{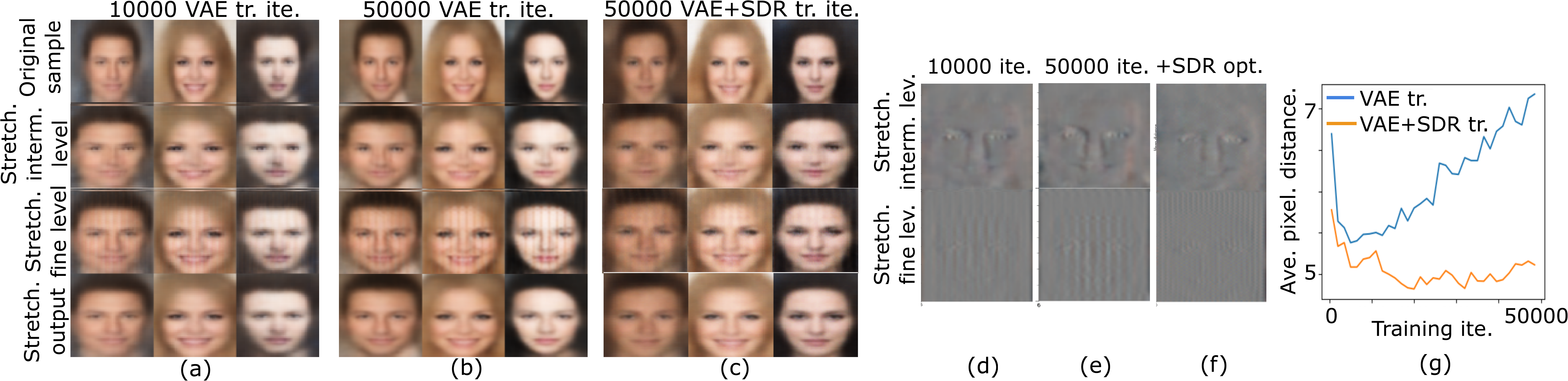}
	\caption{VAE stretching extrapolations. (a-c) VAE extrapolation samples (a-b classical VAE training, c using VAE interleaved with SDR optimization, see text). (d-f) Pixel difference between stretched output and extrapolated samples. (g) Evolution of MSE when stretching at the fine level, with and without {SDR optimization}. See also Suppl. Fig.~3.
		\label{fig:distort}}
\end{figure*}
The extrapolated images obtained by distorting convolutional layers' activation maps are presented in the two middle rows of  Fig.~\ref{fig:distort}a for the VAE trained with 10000 iterations. Note the top and bottom rows respectively correspond to the original output samples, and the result of trivially applying stretching directly to them (these are only provided for comparison with respect to extrapolated samples). This affects differently features encoded at different scales of the picture:  stretching the intermediate level activation
maps (second row of Fig.~\ref{fig:distort}a) mostly keeps the original
dimensions of each eye, while inter-eye distance stretches in line with extrapolation ability that we introduced in Sec.~\ref{sec:extrapol} (Model~\ref{model:LTLCNN}). This suggests that the detrimental effect of SGD optimization investigated in Sec.~\ref{sec:optToy} did not affect this layer. One plausible interpretation of this good extrapolation behavior is the fact that, in contrast with our toy examples, the intermediate level layer contains a large number of channels trained in parallel trough backpropagation. This may limit the propensity of the overparameterized solutions associated to a single channel to drift, due to the multiple pathways exploited during optimization. In contrast, extrapolation of the fine level activation
maps (second row of Fig.~\ref{fig:distort}a), results in slight vertical artifacts; a weaker extrapolation capability possibly related to the smaller number of channels in this layer.  Interestingly, Fig.~\ref{fig:distort}b replicating the result but after 40000 additional training iterations shows perturbed images of poorer quality for this layer. This suggests, as predicted in Section~\ref{sec:optToy}, a decrease of extrapolation capabilities with excessive training, as the drifting regime shown in Fig.~\ref{fig:eyeTrajectories} takes over. In particular, stronger periodic interference patterns like in Fig.~\ref{fig:scaleEyeIllust} (bottom row) appear for the stretching of the fine level hidden layer, which comprises fewer channels, and are thus likely to undergo an earlier drift regime (compare Figs.~\ref{fig:distort}b vs.~\ref{fig:distort}a, 3rd row). To quantify this effect, we tracked the evolution (as the number of iterations grows) of the mean square errors for the complete picture (Fig.~\ref{fig:distort}g), resulting from the stretch of the fine level convolutional layer. This difference grows as the training progresses and the same trend can be observed for the mean squared error of the complete picture.

We next investigated whether enforcing more $\mathcal{S}$-genericity between layers during optimization can temper this effect.
We trained a VAE by alternatively minimizing spectral dependence of eq.~(\ref{eq:proxy}) at image, fine
and intermediate levels, interleaved with one SGD iteration on the VAE objective. Fig.~\ref{fig:distort}c,g show a clear effect of spectral independence minimization on limiting the increase in the distortions as training evolves. This is confirmed by the analysis of pixel difference for 50000 iterations, as seen in Fig.~\ref{fig:distort}f: perturbations of the
intermediate and fine level exhibit better localization, compared to what was obtained at the same number of iterations (Fig.~\ref{fig:distort}e) with classical VAE training, supporting the link between extrapolation and $\mathcal{S}$-genericity of Sec.~\ref{sec:scaleIM}. See also App.~B.13.
%

\michel{reoved stretch of specif feature for now}
\subsection{Genericity of Encoder versus Decoder}
The above qualitative results suggest that extrapolation capabilities 
are observable to some extent in vanilla generative architectures (the decoder of a VAE), but vary depending on the layer considered and can be improved by SDR optimization. We complement these qualitative observations by a validation of the genericity assumption based on the comparison with "inverse" architecture (the encoder of a VAE, see App.~B.10), 
 in line with Sec.~\ref{sec:anticausal}. We study the distribution of the SDR statistic between all possible (filter, activation map) pairs in a given layer. The result for the VAE is shown in Fig.~\ref{fig:histSICVAEtrainingComp},
exhibiting a mode of the SDR close to 1 - the value of ideal spectral independence - for layers of the decoder, which suggests genericity of the convolution kernels between successive layers. 
Interestingly, the encoder, which implements convolutional layers of the same
dimensions in reverse order, exhibits a much broader distribution of the SDR at all levels,
especially for layers encoding lower level image features.  This is in line with results stating 
presented in Sec.~\ref{sec:anticausal}, 
that if a mechanism (here the generator) satisfies the principle of independent causal mechanisms, the inverse mechanism (here the encoder) will not \citep{shajarisale2015}. 
In supplemental analysis, (App.~B.13, 
 Suppl. Fig.~5), 
 we performed the same study on GANs. 


\begin{figure}
	\includegraphics[width=1\linewidth]{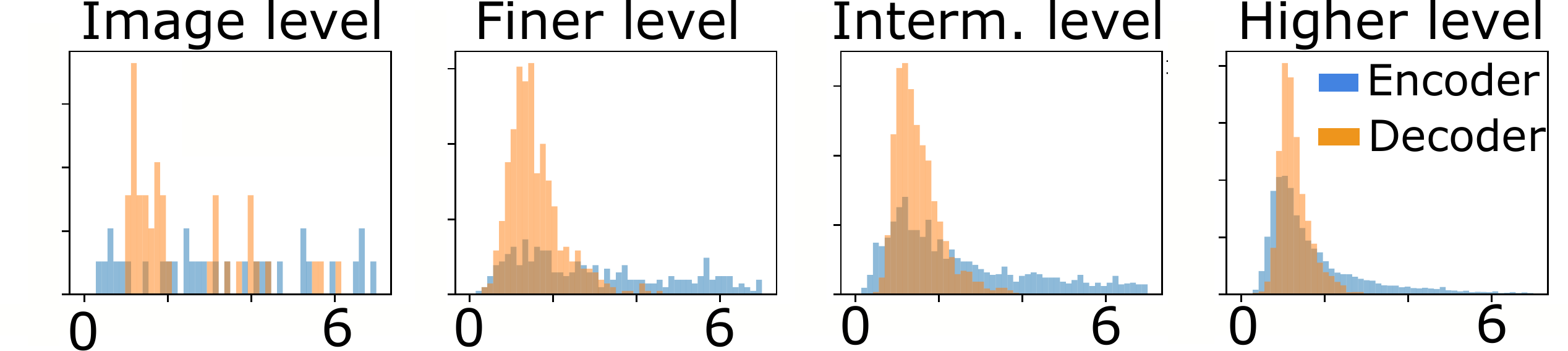}
	\caption{Superimposed SDR histograms of trained VAE decoder and encoder for different hidden layers.\label{fig:histSICVAEtrainingComp}}
\end{figure}

\textbf{Conclusion.}
Our framework to study extrapolation abilities of multi-layered generators based on \textit{Independence of Mechanisms} replaces causal identifiability by a milder constraint of genericity, and shows how SGD training may be detrimental to extrapolation. 
Experiments are consistent with these insights and support spectral independence is a interesting indicator of IM in convolutional generative models. 
This provides insights to train statistical models that better capture the mechanisms of empirical phenomena. 

\newpage
\section*{Ethical Impact}
Although this work is mostly theoretical and conceptual, we anticipate the following impact of this research direction. First, our work addresses how to enforce a causal structure in generative models trained from data. This allows developing statistical models that can better capture the outcomes of previously unseen perturbations to the system that generated the data, and as a consequence can have a positive impact on our ability to learn from observed data in context where experiments are impossible for ethical and practical reasons. Our focus on the notion of extrapolations is particularly suited to be investigate unprecedented climatic, economical and societal challenges facing humankind in the near future. Additionally, augmenting the learning algorithms of artificial systems with causal principles may allow more autonomy and robustness when facing novel environment, possibly leading to both positive and negative societal outcomes. Our approach however proposes a way to understand, formulate and control what kind of robustness should or should not be enforced, providing decision makers with information to guide their choices.

\bibliographystyle{aaai21}

{\small 
\bibliography{ganSIC}}

\onecolumn
\newpage
\addcontentsline{toc}{section}{Appendix} 
\part{Technical appendix\\ \huge \itshape A Theory of Independent Mechanisms for Extrapolation \\ in Generative Models} 
\parttoc 
\appendix
\setcounter{figure}{0}   


\captionsetup[figure]{labelfont={bf},name={Supplemental Figure},labelsep=period}

\section*{General information}
\addcontentsline{toc}{section}{General information}
\subsection*{Computational tools}
We provide code for the experiments in the archive \url{https://gitlab.tuebingen.mpg.de/besserve/code-repository/-/blob/main/zip-files/AAAI21code.zip}. Routines associated to each figure can be found in the subfolder with the corresponding name and are detailed in the sections below.

\subsubsection*{Core routines and dependencies}
The python routines require \textit{numpy}, \textit{tensorflow} and multiple associated standard libraries detailed at the beginning of each routine. In addition, libraries provided by authors of \citet{besserve2018counterfactuals} and custom libraries for this project are provided in the \textit{libs} sub-folders.

\subsubsection*{Deep networks and associated hyperparameters}
Links to implementations of the generative models we used (VAE and GANs) are provided in main text. We base our analysis on checkpoints resulting from the training of each model (see Appendix~\ref{app:modelhyp} for architecture details). All hyperparameters for the analyses of these networks are provided in the hyperparameter files that can be found at the root of the \textit{libs} sub-folder of \textit{code.zip}. The choice of hyperparameters reflects the choices made by the authors of the original implementation, as well as our effort to achieve the best trade-off between visual quality and image diversity (e.g. choice of the $\beta$ parameter in the $\beta$-VAE). In the context of paths provided in the hyperparameter files, the \textit{celebA} dataset as well as checkpoint folders are assumed to be located in the \texttt{/data} folder, which needs to modified to match the local computer settings.

%
%
%
%
%
%

\subsection*{Computing infrastructure}
\addcontentsline{toc}{subsection}{Computing infrastructure}
Training of generative models was performed on a cluster with the following properties:
\begin{itemize}
	\item 42 Nvidia K20x GPGPUs,
	\item 80 Nvidia K80 GPGPUs,
	\item 8 Nvidia P40 GPGPUs,
	\item 20 Nividia P100 GPGPUs,
	\item 24 Nvidia V100 GPGPus,
	\item HTCondor as scheduler.
\end{itemize} 
While the use of the cluster was useful to reduce training time, training was occasionally performed on the following desktop architecture in a matter of few days.

Baring the training of the generative models, analysis was run on a desktop with the following characteristics (no GPU was exploited):
\begin{itemize}
	\item CPU: Quad core Intel Core i7-4790K, cache: 8192KB, clock speeds: max: 4400 MHz.
	\item  Memory: 32Gb.
\end{itemize}

\newpage

\subsection*{Main text symbols and acronyms}
\addcontentsline{toc}{subsection}{Main text symbols}
\captionof{table}{\textbf{List of symbols and acronyms.}\label{tab:symbols}}
{\hfill
\begin{tabular}{llc}
	\scriptsize	Abbrev./Symbol & Name & Eq.
	\\\hline
	NF & Normalizing Flow &\\
	$\circ$ & function composition& \ref{eq:COS}\\
	$\ostar$ & circular convolution &\ref{eq:circonv} \\
	$\boldsymbol{\theta}^*$ & true parameters&\\
	$S_{\boldsymbol{\theta}^*}$ & solution set & \ref{eq:solSet} \\
	COS / $S^\Omega_{\boldsymbol{\theta}^*}$ & Composed Over-parametrization Set &\ref{eq:COS}\\
	$\mathcal{G}$ & group of transformations for extrapolations & \ref{eq:metaclass}\\
	$\mathcal{M}^{\mathcal{G}}_{\boldsymbol{\theta}^*}$ & extrapolated class & \ref{eq:metaclass}\\
		$\mathcal{S}$ & multiplicative group implementing stretching & Sec.~\ref{sec:metaclass}\\
		$\varphi$ & contrast & \ref{eq:genRatio}\\
	$\rho$ & Generic/Spectral Density Ratio (SDR) & \ref{eq:genRatio}/\ref{eq:sdr}\\
	$\odot$ & entrywise product & \ref{eq:circonv}
\end{tabular}\hfill}

\section{Proofs of main text propositions}\label{sec:proofs}


\subsection{Proof of Proposition~\ref{prop:COSequalS}}\label{sec:COSequalSproof}

First, we prove $\Omega$ is a group. Let $\Omega_1$ be the set of mappings $\omega: \mathcal{V}\rightarrow\mathcal{V}$ such that $\omega \circ \mathcal{F}_1\subset \mathcal{F}_1$. Then each $\omega\in \Omega_1$ takes the form $\omega = f_1^{-1}\circ f_1'$  for some $f_1,f_1'\in \mathcal{F}_1$ (because the model is NF). This trivially implies that $\Omega_1$ is a subgroup of the bijections  $\mathcal{V}\rightarrow\mathcal{V}$. In the same way, the set $\Omega_2$ of mappings $\omega: \mathcal{V}\rightarrow\mathcal{V}$ such that $ \mathcal{F}_2 \circ \omega \subset \mathcal{F}_2$ is also a subgroup and $\Omega = \Omega_1\cap\Omega_2$ is then also a subgroup, thus it is a group.

Next, we prove $S_{(f_1^*,\,f_2^*)}^\Omega = S_{(f_1^*,\,f_2^*)}$. From its definition $S_{(f_1^*,\,f_2^*)}^\Omega \subseteq S_{(f_1^*,\,f_2^*)}$ is trivial. It remains to prove $S_{(f_1^*,\,f_2^*)} \subseteq S_{(f_1^*,\,f_2^*)}^\Omega$. Assume $(f_1,\,f_2)\in S_{(f_1^*,\,f_2^*)}$, then $f_2\circ f_1= f_2^* \circ f_1^*$ by the functional identifiability assumption.

As a consequence, $f_2 = f_2^*\circ \omega $ and $f_1 = \omega^{-1} \circ f_1^*$ for $\omega = f_1^*\circ f_1^{-1}= (f_2^{-1}\circ f_2^*)^{-1}$. From these expression, it is clear that $\omega \in  \Omega_1\cap\Omega_2 = \Omega$, which implies $(f_1,\,f_2)\in S_{(f_1^*,\,f_2^*)}^\Omega$\,.
\qed

\subsection{Proof of Proposition~\ref{prop:convFaith}}
We do the proof in the 1D case, which trivially generalizes to 2D images without fundamental differences. Given $d$ prime number, the multiplicative group of non-zero integers modulo $d$ is cyclic (see App.~\ref{app:stretchGrp}), there exists a $(d-1)\times (d-1)$ permutation matrix $P$ such that the stretching action of integer $g$ on $\widehat{k}_1$ the DFT of kernel $k_1$, turns it into 
\begin{equation}\label{eq:stretchRep}
g\cdot \widehat{k_1}= \left[
\begin{matrix}
P_\sigma^g &0& \boldsymbol{0}\\
\boldsymbol{0} &1& \boldsymbol{0}\\
\boldsymbol{0} & 0&P^g
\end{matrix}\right]
 \widehat{k_1}
\end{equation}
where $P_\sigma$ is the permutation matrix corresponding to $P$ with rows and columns in reverse order (to model the action on negative integers).

Now, assume $(\widehat{k_1},\widehat{k_2})$ is $\mathcal{S}$-equivalent to  $(\widehat{k_1}^*,\widehat{k_2}^*)$, then there exist a permutation $s$ of the elements of $\mathcal{S}$ such that for all $g\in \mathcal{S}$
\[
g\cdot \widehat{k_1} \odot \widehat{k_2}=s(g)\cdot \widehat{k_1}^* \odot \widehat{k_2}^*\,.
\]
Averaging over all values of $g$, we get that
\[
\frac{1}{d-1}\left[
\begin{matrix}
\boldsymbol{1} &0& \boldsymbol{0}\\
\boldsymbol{0} &1& \boldsymbol{0}\\
\boldsymbol{0} & 0&\boldsymbol{1}
\end{matrix}\right]\cdot \widehat{k_1} \odot \widehat{k_2}=
\frac{1}{d-1}\left[
\begin{matrix}
\boldsymbol{1} &0& \boldsymbol{0}\\
\boldsymbol{0} &1& \boldsymbol{0}\\
\boldsymbol{0} & 0&\boldsymbol{1}
\end{matrix}\right]\cdot \widehat{k_1}^* \odot \widehat{k_2}^*\,,
\]
where $\boldsymbol{1}$ are matrices with all coefficients equal to one. Then for all strictly positive frequencies $\nu$ (strictly negative ones lead to the conjugate expression), this leads to 
\[
\widehat{k_2}[\nu]=\frac{\left\langle \widehat{k_1}^* \right\rangle}{\left\langle \widehat{k_1} \right\rangle}\widehat{k_2}^*[\nu]=\lambda \widehat{k_2}^*[\nu]\,,
\]
where $\left\langle .\right\rangle$ denotes averaging over positive frequencies. Both pairs being solutions of the problem, this implies also
\[
\widehat{k_1}[\nu]=\lambda^{-1} \widehat{k_1}^*[\nu]\,.
\]
The converse implication is straightforward.
\qed

\subsection{Sketch of the proof of Proposition~\ref{prop:toySIC}}
The expression	in eq.~(\ref{eq:sdr}) can be derived using the same principles as in Prop.~\ref{prop:convFaith}. Briefly, we can write the denominator of the generic ratio as:
\[
\mathbb{E}_{g\in\mu_{\mathcal{S}}} \left\langle  |\widehat{k}_2\odot\left(g\cdot \widehat{k}_1\right)  \odot \widehat{\boldsymbol{Z}}|^2 \right\rangle= \left\langle  |\widehat{k}_2\odot\mathbb{E}_{g\in\mu_{\mathcal{S}}}\left[g\cdot \widehat{k}_1\right]  \odot \widehat{\boldsymbol{Z}}|^2 \right\rangle
\]
 As $\widehat{\boldsymbol{Z}}$ as constant modulus 1, and as integrating over the Haar measure $\mu_{\mathcal{S}}$ of such discrete group corresponds to averaging over all (finite) group elements, we can exploit the representation of eq.~(\ref{eq:stretchRep}) to obtain the result.

Next, to prove $\mathcal{S}$-genericity, we observe that realizations of $V= k_1\ostar Z$ are discrete 2D images consisting in unit discrete Diracs located at two different pixels.
Without loss of generality, we work on one realization $v$ and assume one of these pixels is located at coordinate $(0,\,0)$ and the other at coordinate $(m_0,\,n_0)$.
Then the squared Discrete Fourier Transform (DFT) of $f$ writes
\[
\left|\mathbf{F}v(u,v)\right|^2=\frac{1}{d^2}\left(2+2\cos(2\pi (m_0 u+n_0 v))\right)
\]
and its sum over frequencies is 
\[
\left\langle\left|\mathbf{F}v(u,v)\right|^2\right\rangle = 2
\]

\begin{equation}
\label{eq:sdrSI}
\rho{(\boldsymbol{V}, {k_2})}=\frac{\langle\mathbb{E} |\widehat{k}_2\odot \widehat{\boldsymbol{V}}|^2\rangle  }{\langle \mathbb{E}|  \widehat{\boldsymbol{V}}|^2\rangle \langle |\widehat{k}_2|^2 \rangle }=\frac{\langle |\widehat{k}_2\odot \widehat{k}_1|^2\rangle  }{\langle |\widehat{k}_1|^2\rangle \langle |\widehat{k}_2|^2 \rangle }\,,
\end{equation}

Without loss of generality we assume that $k_2$ has unit energy, such that the SDR writes
\[
\rho = \frac{\left\langle\left|\mathbf{F}k_2(u,v)\mathbf{F}g(u,v)\right|^2\right\rangle}{\left\langle\left|\mathbf{F}g(u,v)\right|^2\right\rangle} = \left\langle\left|\mathbf{F}k_2(u,v)\mathbf{F}g(u,v)\right|^2\right\rangle/2
\]
Using the Fourier convolution-product calculation rules, this terms also corresponds to the value of circular convolution $(k_2\ostar k_{2,\sigma}) \ostar (v\ostar v_\sigma)/2$ at index $(0,0)$, where $k_{2,\sigma}$ denotes mirroring of both spatial axes. Since the support of $f$ is bounded by a square of side $\delta$, the support of $(k_2\ostar k_{2,\sigma})$ is bounded by a square of side $2\delta$. Convolution of this quantity by $(v\ostar v_\sigma)/2$ (which is a sum of one central unit Diracs at $(0,\,0)$ and two side Diracs at $\pm (m_0,\,n_0)$), yields a superposition of one central pattern  $(k_2\ostar k_{2_\sigma})$ and two translated versions of this term around $\pm (m_0,\,n_0)$, and this pattern is additionally periodized with period d along both dimensions (due to circularity of the convolution). Since by assumption $2\delta<\max(m_0,\,n_0,\,d-m_0,\,d-n_0)$, then the supports of the translated terms, as well as their periodized copies, do not reach index $(0,0)$, and the value of $\rho$ is given by the central term
\[
\rho = k_2\ostar k_{2,\sigma}(0,\,0) = 1\,,
\]
due to the unit energy assumption on $k_2$.
\qed

\section{Additional methods}\label{sec:methods}
\subsection{\textit{FluoHair} experiment}\label{app:fluohair}
To obtain the result of Fig.~\ref{fig:fluohair}, we proceeded as follows.
We ran the clustering of hidden layer channels into modules encoding different properties, using the approach proposed by \citet{besserve2018counterfactuals} using the non-negative matrix factorization technique and chose a hyperparameter of 3 clusters. For the last hidden layer of the generator, we identified the channels belonging to the cluster encoding hair properties. We then identified and modified the tensor encoding the convolution operation of the last layer (mapping the last hidden layer to RGB image color channels, as described in Fig.~\ref{fig:fluohair}), by changing the sign of the kernel coefficients corresponding to inputs originating form the identified hidden channels encoding hair. This generates pink hair. In order to change color to green (or blue), for the same coefficients, we permute in addition the targeted color channels (between red, green and blue).
\michel{add details on the used architecture}

\subsection{Background on group theory}\label{sec:groupth}
 We introduce concisely the concepts and results of group theory necessary to this paper. The authors can refer for example to \citep{tung1985group,wijsman1990invariant,Eaton1989} for more details.
\begin{defn}[Group]
	A set $\mathcal{G}$ is said to form a group if there is an operation `*', called group multiplication, such that:
	\begin{enumerate}
		\item For any $a,b\in \mathcal{G}$, $a*b\in \mathcal{G}$.
		\item The operation is \textit{associative}: $a*(b*c)=(a*b)*c$, for all $a,b,c\in\mathcal{G}$,
		\item  There is one identity element $e\in\mathcal{G}$ such that, $g*e=e$ for all $g\in\mathcal{G}$,
		\item  Each $g\in\mathcal{G}$ has an inverse $g^{-1}\in\mathcal{G}$ such that, $g*g^{-1}=e$.
	\end{enumerate}
	A subset of $\mathcal{G}$ is called a subgroup if it is a group under the same multiplication operation.
\end{defn}
The following elementary properties are a direct consequence of the above definition:
$e^{-1}=e$,
$g^{-1}*g=e$,
$e*g=g$, for all $g\in\mathcal{G}$.
\newcommand{\G}{\mathcal{G}}
\newcommand{\inv}[1]{{#1}^{-1}}
\begin{defn}[Topological group]
	\label{def:topogroup}
	A locally compact Hausdorff topological group is a group equipped with a locally compact Hausdorff topology such that:
	\begin{itemize}
		\item $\G\rightarrow \G: x\mapsto \inv{x}$ is continuous,
		\item $\G\times\G\rightarrow \G: (x,y)\mapsto x.y$  is continuous (using the product topology).
	\end{itemize}
	The $\sigma$-algebra generated by all open sets of G is called the Borel algebra of $\G$. 
\end{defn}

\begin{defn}[Invariant measure]
	Let $\G$ be a topological group according to definition~\ref{def:topogroup}. Let $K(\G)$ be the set of continuous real valued functions with compact support on $\G$. A radon measure $\mu$ defined on Borel subsets is left invariant if for all $f\in K(\G)$ and $g\in\G$
	$$
	\int_G f(\inv{g} x)d\mu(x)=\int_G f(x)d\mu(x)
	$$
	Such a measure is called a \textit{Haar measure}.
\end{defn}
A key result regarding topological groups is the existence and uniqueness up to a positive constant of the Haar measure \citep{Eaton1989}. Whenever $\G$ is compact, the Haar measures are finite and we will denote $\mu_\G$ the unique Haar measure such that $\mu_\G(\G)=1$, defining an invariant probability measure on the group.

\subsection{Background on circular convolution}\label{sec:circconv}
We provide first the definition for a one dimensional signal.

Circular convolution of finite sequences and their Fourier analysis are best described by considering the signal periodic. In our developments, whenever appropriate, the signal ${ a}=\{a[k] , k\in [0,\,d-1]\}$ can be considered as $d$-periodic by defining for any $k\in\mathbb{Z}$, $a[k]=a[k']$, whenever $k'\in[0,\,d-1]$ and $k=k'(mod d)$. One way to describe these sequences is then to see them as functions of the quotient ring $\mathbb{Z}_d=\mathbb{Z}/d\mathbb{Z}$.

Given two $d$-periodic sequences ${ a}=\{a[k] , k\in \mathbb{Z}_d\}$, ${ b}=\{a[k] , k\in \mathbb{Z}_d\}$ their circular convolution is defined as
\[
a\ostar b [n]= \sum_{k=1}^d a[k]b[n-k]\,.
\]

Generalization to 2 dimensions is straightforward by periodizing the image along both dimensions, and then applying the 2D formula:
\[
a\ostar b [n,m]= \sum_{k=1}^d a[k,j]b[n-k,m-j]\,.
\]

\subsection{Background on Fourier analysis of discrete signals and images}\label{sec:dft}
The Discrete Fourier Transform (DFT) of a periodic sequence ${ a}=\{a[k] , k\in \mathbb{Z}_d\}$ is defined as
$$
\widehat{\rm a}[n]=\sum_{k\in \mathbb{Z}_d} a[k] e^{-\mathbf{i}2\pi n k/d},\, n \in \mathbb{Z}_d\,.
$$
Note that the DFT of such sequence can as well be seen as a $d$-periodic sequence. Importantly, the DFT is invertible using the formula
$$
a[k]=\sum_{n\in \mathbb{Z}_d} \widehat{a}[n] e^{\mathbf{i}2\pi n k/d},\, n \in \mathbb{Z}_d\,.
$$

By Parseval's theorem, the energy (sum of squared coefficients) of the sequence can be expressed in the Fourier domain by $\|a\|_{2}^2=\frac{1}{d}\sum_{k\in\mathbb{Z}_d} |a[k]|^2d$. 
The Fourier transform can be easily generalized to 2D signals of the form $\{b[k,l] , (k,l)\in \mathbb{Z}^2\}$, leading to a 2D function, 1-periodic with respect to both arguments
$$
\widehat{\rm b}(u,v)=\sum_{k\in \mathbb{Z},l\in \mathbb{Z}} b[k,l] e^{-\mathbf{i}2\pi(uk+vl)},\, (u,v) \in \mathbb{R}^2\,.
$$

In both the 1D and 2D cases, one interesting property of the DFT is that it transforms convolutions into entrywise products. This writes, for the 1D case
$$
\widehat{a\ostar b}=\widehat{\rm a}\cdot\widehat{\rm b}=\{ \widehat{\rm a}[k]\cdot\widehat{\rm b}[k],\,k\in \mathbb{Z}_d\}\,.
$$

In the case of Model~\ref{model:LTLCNN}, this leads to
\begin{equation}\label{eq:circonvSI}
	\widehat{\boldsymbol{X}}=\widehat{k_2 }\cdot \widehat{\boldsymbol{V}}=\widehat{k_2}\cdot \widehat{k_1}\cdot \widehat{\boldsymbol{Z}}\,.
\end{equation}
Now if we compute the power over non-constant frequencies, since the Discrete Fourier Transform (DFT) of $\boldsymbol{Z}$ has modulus one at all frequencies (as dirac impulse), Parseval theorem yields
\begin{equation}
\label{eq:powSI}
\mathcal{P}(\textbf{X})=\frac{1}{d^2}\sum_{i,j} |\widehat{k}_2(i,j)\widehat{k}_1(i,j)|^2=\left\langle |\widehat{k}_2\odot\widehat{k}_1|^2\right\rangle\,,
\end{equation}
where $\langle .\rangle$ denotes averaging over 2D frequencies and $\odot$ is the entrywise product.

\subsection{Background on the discrete stretching group $\mathcal{S}$}\label{app:stretchGrp}
We consider $d$ prime number and the set of integers modulo $d$, $\mathbb{Z}_d=\{0,1,\dots,d-1\}$. Then the subset of non-zero integers $\mathcal{S}=\mathbb{Z}_d^*=\{1,\dots,d-1\}$ equipped with multiplication (modulo $d$), is a commutative cyclic group \footnote{\url{https://en.wikipedia.org/wiki/Multiplicative_group_of_integers_modulo_n}}. As a consequence, we can consider the action of $g\in \mathbb{Z}_d^*$ on the t-uple of signed indices 
$$
\{-d+1,\dots,-1,0,1,\dots,d-1\}\,,
$$
resulting in the t-uple 
$$
\{-g\cdot(d-1),\dots,-g\cdot 1,0,g\cdot 1,\dots,g\cdot (d-1)\}\,.
$$
Because of the multiplicative group structure of $\mathbb{Z}_d^*$, this action operates a permutation of the strictly positive and strictly negative indices, while index zero remains unchanged. Also because of the group structure, this permutation of the t-uple can be inverted by the action of the inverse $g^{-1}$ in $\mathbb{Z}_d^*$.

\michel{add classical results here (wikipedia...)}

\subsection{Additional results for Section~\ref{sec:extrapol}}\label{sec:addresultsSec2}
 \begin{corol}\label{prop:convCOS}
 	Model~\ref{model:LTLCNN} is NF, such that $S_{k_1^*,k_2^*}=S^\Omega_{k_1^*,k_2^*}$ with $\Omega$ the set of invertible $2d+1\times 2d+1$ convolution kernels.
 \end{corol}
 \begin{proof}
 	Taking the steps of the proof of Proposition~\ref{prop:COSequalS}, it is easy to see that $\mathcal{F}_1=\Omega_1=\mathcal{F}_2=\Omega_1$. As a consequence it is also the intersection $\Omega$.
 \end{proof}

 \begin{prop}\label{prop:linFaith}
	Let $\mathcal{S}=\{1,2,...,d-1\}$ be the multiplicative group of integers modulo $d$, $d$ prime number. We index the matrix coordinate from 1, leading to index $\{1,2,...,d-1\}$, such that $g\in \mathcal{S}$ acts on diagonal matrix $A$ by replacing diagonal component of coordinate $k$ by component of coordinate $g\cdot k$. A solution $(A,\,B)$ for Model~\ref{model:ABmodel} is $\mathcal{S}$-equivalent to true model $(A^*,\,B^*)$ if and only if there is a $\lambda>0$ such that $(A,\,B)=(\lambda^{-1}A^*,\, \lambda B^*)$.
\end{prop}
\begin{proof}[Proof sketch]
This follows the steps of Prop.~\ref{prop:convFaith} using the diagonal elements of the matrices instead of the DFT of convolution kernels, and using only the second half of the components (strictly postive frequencies).
\end{proof}

\begin{corol}\label{prop:linCOS}
	Model~\ref{model:ABmodel} is an NF model for which $\Omega$ is the set of square positive definite diagonal matrices, and $S^\Omega_{A^*,B^*}=S_{A^*,B^*}$.
\end{corol}
\begin{proof}
	Let us first characterize $\Omega$.
	Assume $\omega\in \Omega$, then $A\circ\omega=A'$ for $A,A'$ linear transformations associated to diagonal positive definite matrices. Thus $\omega=A^{-1}A'$ is also the canonical linear transformation of a diagonal positive definite matrix. Conversely, assume $\omega$ is diagonal positive definite, then obviously $A\circ\omega$ and $\omega^{-1}\circ B$ are also positive definite. Hence $\Omega$ is (canonically associated to) the group of square diagonal positive definite matrices.
	
	Second, let us show $S_{(A^*,B^*)}\in S_{(A^*,B^*)}^\Omega$ (converse inclusion holds by definition). 
	Let $({A},{B})\in S_{(A^*,B^*)}$, then the $k$-th diagonal coefficient satisfies
	\[
	AB = A^*B^* \triangleq C^*
	\]
	which implies $A = C^*B^{-1}$ and $B = A^{-1}C^*$, which leads to $A=A^*\omega$ and $B=\omega^{-1}B^*$ for $\omega=B^*B^{-1}\in \Omega$. 
	
	Thus $(A,B)\in S_{(A^*,B^*)}^\Omega$.
\end{proof}

\subsection{Analysis of continuous overparameterized time gradient descent}\label{app:CTGD}
\begin{prop}\label{prop:CTGD}
	Consider the CTGD of problem~(\ref{eq:toyLoss}), from any initial point $(a_0,\,b_0)>0$ the trajectory leaves the quantity $L(a,b)=a^2-b^2$ unchanged and 
	converges to the intersection point with  $S_c=S^\Omega_c=\{(a,c/a),a>0\}$. 
\end{prop}

\begin{proof}
	The gradient for objective in equation~\ref{eq:toyLoss} is
	\begin{eqnarray}
	\nabla_a \mathcal{L} &=& -2b(c-ab)\\
	\nabla_b \mathcal{L} &=& -2a(c-ab)
	\end{eqnarray}
	Hence the dynamics of continuous time gradient descent is (assuming a unit learning rate without loss of generality)
	\begin{eqnarray}
	\frac{da}{dt} = -\nabla_a\mathcal{L} &=& 2b(c-ab)\\
	\frac{db}{dt} = -\nabla_b\mathcal{L} &=& 2a(c-ab)
	\end{eqnarray}
	
	Thus the trajectories of this dynamical system satisfy the equation
	\[
	b\frac{db}{dt} = a\frac{da}{dt}\,,
	\]
	implying that $L(a(t),b(t)) = a(t)^2 - b(t)^2$ is a constant along the trajectories. Assuming $b(t)>0$, each trajectory satisfies $b(t) = \sqrt{a^2(t)+D}$ for some constant $D$.
	
	If we restrict ourselves to the domain $b>0$ and $a>0$, stationary points are the element of the hyperbola $S_0=\{(a,b)\in \mathbb{R}^+ \times \mathbb{R}^+,\, b=c/a\}$.
	\qed
\end{proof}

\subsection{Analysis of SGD drift}\label{app:SGD}
We consider SGD as Algorithm~\ref{alg:sgd1} 
\begin{algorithm}
	Sample $c_n$ from $P_C$\;
	$a_{n+1} \Leftarrow a_{n}-\lambda\nabla_a\ell(c_n;\,(a_n,b_n))$\;
	$b_{n+1} \Leftarrow b_{n}-\lambda\nabla_b\ell(c_n;\,(a_{n},b_n))$\;
	\caption{SGD\label{alg:sgd1}}
\end{algorithm}
We have the following result.
\begin{prop}\label{prop:SGDbias}
	Assume an initial distribution $A^{(0)}\sim\mathcal{N}(a_0,\sigma'^2)$ and $B^{(0)}\sim\mathcal{N}(b_0,\sigma'^2)$, such that $(a_0,\,b_0)\in S_{c_0}$
	, then after one SGD iteration, the updated values $(A^{(1)},B^{(1)})$ satisfy (using $L(a,b)=a^2-b^2$ as above)
	$$
	\mathbb{E}[L(A^{(1)},B^{(1)})]=
	\eta\mathbb{E}[L(A^{(0)},B^{(0)})]\,,\,0<\eta<1.
	$$
\end{prop}
\begin{proof}
	The evolution of $L(a,b)$ during SGD follows the difference equation
	\[
	L(a_{n+1},b_{n+1})=a_{n+1}^2-b_{n+1}^2=(a_{n}^2-b_{n}^2)(1-4\lambda^2(c_n-a_{n}b_{n})^2)
	\]
	by expanding the left hand side and simplifying the expression by exploiting the independence and Gaussianity of $c_n$, $a_{n}$ and $b_{n}$ we get the result.
	\qed 
\end{proof}

\subsection{Asynchronous gradient descent}\label{sec:asgd}
We make a slight change in the gradient update of section~\ref{sec:optToy}  according to Algorithm~\ref{alg:sgd2}, making it asynchronous by updating $a$ before computing the gradient with respect to $b$.
\begin{algorithm}
	Sample $c_n$ from $P_C$\;
	$a_{n+1} \Leftarrow a_{n}-\lambda\nabla_a\ell(c_n;\,(a_n,b_n))$\;
	$b_{n+1} \Leftarrow b_{n}-\lambda\nabla_b\ell(c_n;(\textcolor{red}{a_{n+1}},\,b_n))$\;
	\caption{Asynchronous SGD (ASGD)\label{alg:sgd2}}
\end{algorithm}

Interestingly, the resulting dynamic is again different from both previous
cases. The trajectory drifts along $S_0$ reaching asymptotically the $a$ axis
(see blue simulated sample path Fig.\ref{fig:toyTrajectories}). This generates a
systematic dependency between $a$ and $b$ as the optimization evolves, which is
a property influenced by the detailed optimization procedure. Note that this
form of asynchronous update may be implemented in an actual deep learning algorithm for the sake of parallelization and computational efficiency.

\subsection{SDR analysis of deep generative models}\label{app:deepmethods}

We consider two successive layers. As show on Fig.~\ref{fig:pathways}, a difference with Model~\ref{model:LTLCNN} studied in previous sections, a single layer consists of multiple 2D activation maps, called channels, to which are applied convolutions and non-linearities to yield activations maps forwarded to the next layer.
More precisely, an activation map $\boldsymbol{X}$ (corresponding to one channel in the considered layer) is generated from the $n$
channels' activation maps $\boldsymbol{V}=(\boldsymbol{V}_1,\dots,\boldsymbol{V}_n)$ in the previous layer through the multi-channel kernel $k_2=(k_2^1,\dots, k_2^n)$ as
\begin{align}
\label{eq:convolution}
X &= \sum_{i=1}^n k_2^i \ostar V_i + b\,. 
\end{align}
By looking only at ``partial'' activation map $X=k_2^i \ostar V_i$, 
we get back to the case of Model~\ref{model:LTLCNN} (up to some additive constant bias). Therefore, unless
specified otherwise, the term \textit{filter} will refer
to partial filters $(f^i)$.

Next, to get an empirical estimate of SDR for the partial filter, we consider its cause-effect formulation in eq.~\ref{eq:sdr} and we estimate the expectation with an empirical average of the batch of samples $(v_i^1,...,v_i^B)$ of size $B$.
\begin{equation}
\label{eq:sdremp}
\rho{(\boldsymbol{V}_i, {k_2^i})}=\frac{\langle\mathbb{E} |\widehat{k}_2^i\odot \widehat{\boldsymbol{V}_i}|^2\rangle  }{\langle \mathbb{E}|  \widehat{\boldsymbol{V}_i}|^2\rangle \langle |\widehat{k}_2^i|^2 \rangle }\approx\frac{\langle\frac{1}{B} |\widehat{k}_2^i\odot \widehat{\boldsymbol{v}_i^k}|^2\rangle  }{\langle \frac{1}{B}|  \widehat{\boldsymbol{v}_i^k}|^2\rangle \langle |\widehat{k}_2^i|^2 \rangle }\,,
\end{equation}
\michel{check the sum over batch samples}
One additional difference with respect to Model~\ref{model:LTLCNN} is a stride parameter $k>1$ interleaving $k-1$ zero-value pixels between each input pixels, along each dimension, before applying the convolution operation, in order to progressively increase the dimension and resolution of the image from one layer to the next. 
As shown in App.~\ref{sec:stridedproof}, this can be easily modeled and leads to a slightly different SDR (eq.~\ref{eq:sdrstride}). 

\begin{figure}
	\begin{subfigure}[t]{.48\linewidth}
		\usetikzlibrary{arrows}
\begin{tikzpicture}

\node (f1) at (-4.8,-.1) {\Huge $k_2^1$};
  \node (f) at (-2.1,-.1) {\Huge $k_2^i$};
\node (g) at (-0.9,-.7) {\Huge };
\node (h) at (1.5,-.3) {\Huge };

\node (v6) at (-1,-0.5) {};
\node (z1) at (-4.8,1.7) {\Huge $\boldsymbol{V}_1$};
\node (z) at (.45,1.7) {\Huge $\boldsymbol{V}_i$};
\node (x) at (-4.5,-2.5) {\Huge $\boldsymbol{X}$};

\node (v9) at (0.3,0.5) {};
\node (v9b) at (0.4,0.5) {};
\node (v10) at (-2.3,-1.2) {};
\node (v12) at (-4.5,-1.2) {};
\node (v11) at (0.7,-1.2) {};
\node (v13) at (-5.,0.5) {};
\node (v14) at (-2.5,0.5) {};
\node (v15) at (2.5,0.5) {};

\draw [-triangle 60] (v9b) edge (v10) ;
\draw [-triangle 60] (v9) edge (v11) ;
\draw [-triangle 60] (v9) edge (v12) ;
\draw[opacity=0.5] [-triangle 60] (v13) edge (v12) ;
\draw[opacity=0.5] [-triangle 60] (v14) edge (v12) ;
\draw[opacity=0.5] [-triangle 60] (v15) edge (v12) ;

\node (v1) at (-4.5,-1) {};
\node (v2) at (-4,-0.5) {};
\node (v3) at (-2.5,-0.5) {};
\node (v4) at (-3,-1) {};
\node (v1a) at (-4.5,-1.1) {};
\node (v2a) at (-4,-0.6) {};
\node (v3a) at (-2.5,-0.6) {};
\node (v4a) at (-3,-1.1) {};
\node (v1b) at (-4.5,-1.2) {};
\node (v2b) at (-4,-0.7) {};
\node (v3b) at (-2.5,-0.7) {};
\node (v4b) at (-3,-1.2) {};
\node (v1c) at (-4.5,-1.3) {};
\node (v2c) at (-4,-0.8) {};
\node (v3c) at (-2.5,-0.8) {};
\node (v4c) at (-3,-1.3) {};

\tikzset{
    mygauss/.pic={
       \fill (v1.center) -- (v2.center) -- (v3.center) -- (v4.center) -- cycle ;
 }}
\tikzset{
    mygaussS/.pic={
      \fill (v1.center) -- (v2.center) -- (v3.center) -- (v4.center) -- cycle ;
      \fill[opacity=0.5] (v1a.center) -- (v2a.center) -- (v3a.center) -- (v4a.center) -- cycle ;
      \fill[opacity=0.5] (v1c.center) -- (v2c.center) -- (v3c.center) -- (v4c.center) -- cycle ;
      \fill[opacity=0.5] (v1b.center) -- (v2b.center) -- (v3b.center) -- (v4b.center) -- cycle ;
 }}
    
    \matrix (A) [column sep=5mm, row sep=40] {
 &\pic[fill=green]{mygauss}; & \pic[fill=green]{mygauss}; & \pic[fill=green]{mygauss}; & \pic[fill=green]{mygauss}; \\
 &\pic[fill=red]{mygaussS}; & \pic[fill=red]{mygauss}; & \pic[fill=red]{mygauss}; &\\};

\end{tikzpicture}
		\subcaption{ \label{fig:pathways}}
	\end{subfigure}
	\begin{subfigure}[t]{.4\linewidth}
		\includestandalone[width=\linewidth]{vaeSchema}
		\subcaption{ 
			\label{fig:vaeGenStruct}}
	\end{subfigure}
	\caption{(a) Convolution pathways between successive layers. (b) Architecture of the VAE generator.}
\end{figure}
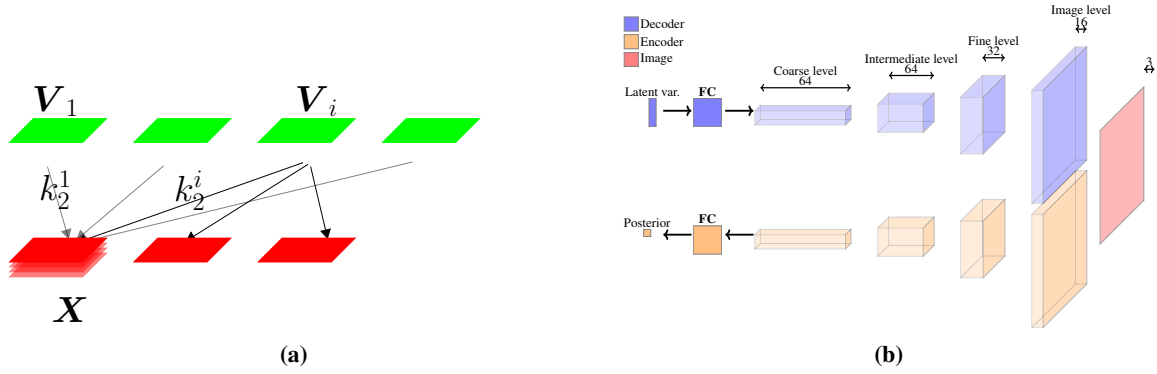

\subsubsection{SDR expression in the strided case}\label{sec:stridedproof}
Striding can be easily modeled, as it amounts to up-sampling the input image before convolution. We denote $.^{\uparrow s}$ the up-sampling operation with integer factor\footnote{$s$ is the inverse of the stride parameter; the latter is fractional in that case} $s$ that turns the 2D activation map $x$ into
\[
\textstyle x^{\uparrow s}[k,l] = \begin{cases}
x[k/s,l/s], &\mbox{$k$ and $l$ multiple of s,}\\
0, &\mbox{otherwise.}
\end{cases}
\]
leading to a compression of the normalized frequency axis in the Fourier domain such that $\widehat{\rm x^{\uparrow s}}(u,v)=\widehat{\rm x}(su,sv)$.
The convolution relation in Fourier domain thus translates to
$\textstyle\widehat{\rm y}(u,v)= \widehat{\rm h}(u,v)\widehat{\rm x}(su,sv)$. As
a consequence, the SDR measure needs to be adapted to up-sampling by rescaling the frequency axis of the activation map with respect to the one of
the filter. Using power spectral density estimates based on  Bartlett's method,
we use a batch of input images of size $B$ leading to $B$ values of activation
map $x$, $x_0,\dots,x_{B-1}$, to obtain the following SDR estimate:
\begin{equation}\label{eq:sdrstride}
\rho_{\{x_i\}\rightarrow f} = \frac{\left\langle \frac{1}{B} \sum_{i=0}^{(B-1)} |{\widehat{\rm f}}(u,v)\widehat{\rm x_i}(su,sv)|^2 \right\rangle }
{\left\langle |\widehat{\rm f} (u,v)|^2 \right\rangle  \left\langle \frac{1}{B} \sum_{i=0}^{(B-1)} |\widehat{\rm x} (u,v)|^2 \right\rangle}\,.
\end{equation}

\subsection{Network hyper-parameters}\label{app:modelhyp}

	\captionof{table}{Default network hyper-parameters (they apply unless otherwise stated in main text).\label{tab:netparams}}
	\begin{tabular}{c|cc}
		Architecture & DCGAN & VAE  \\
		\hline
		Nb. of deconv. layers/channels of generator & 4/(128,64,32,16,1) & 4/(128,64,32,16,3) \\
		Size of activation maps of generator & (4,8,16,32) & (8,16,32,64)  \\
		
		Optimization algorithm & Adam ($\beta=0.5$) &Adam ($\beta=0.5$)\\
		Minimized objective & GAN loss & VAE loss (Gaussian posteriors) \\
		batch size & 64 & 64 \\
		Beta parameter & N/A & 0.0005
	\end{tabular}
	
\subsection{Additional results for Section 3}\label{app:addresOpt}

In order to quantify the effect of the drift induced by SGD as well as the SDR regularization of Sec.~\ref{sec:opt}, we ran 200 simulations with different seeds, for $d=5$. Kernels were initialized with the true solution (normalized to have total energy 1) perturbed by uniformly distributed noise (supported in $[0,1]$). SGD was implemented with a learning rate $\lambda=.01$, and SDR regularization was included by doing an additional gradient descent update of the SDR loss with learning rate $\gamma$ (no SDR optimization amounts to $\gamma=0$). Ability to recover the true solution (up to multiplicative coefficient) was quantified by the absolute cosine distance between the resulting kernels and the true ones in the Fourier domain. 

The results are provided in Fig.~\ref{fig:optimToy}. We observe that, as predicted by the theory in Sec.~\ref{sec:optToy}, the SDR of non-SDR regularized solutions (using only SGD of the least square loss) ultimately drifts towards values largely superior to one, which is prevented by SDR regularization (Fig.~\ref{fig:optimToy}, bottom left panel). Notably, applying an increasingly high SDR regularization drives the SDR faster to 1 (the value for $\mathcal{S}$-genericity). Moreover, the evolution of cosine distances suggest that SDR regularization helps driving the result closer to the true solution (Fig.~\ref{fig:optimToy}, top panels). However, for very large numbers of iteration, quality of the solution decays again, suggesting that SDR regularization may help compensate the drift of solutions only temporarily and that further constraints (e.g. additional regularization with other generic ratio) may be necessary to robustly obtain a beneficial effect of SDR regularization for any number of iterations. These results overall suggest that running SDR regularization for a restricted amount of time may help enforce better extrapolation properties than SGD, which is further supported by our experiments in deep generative models (Sec.~\ref{sec:expDeep}).

\subsection{Additional results for deep models}\label{app:addresDeep}
%

\subsubsection{Multiscale analysis of stretching perturbations}
We also used a discrete Haar wavelet
transform of the images to isolate the contribution of each scale to the image \citep{mallat1999wavelet}. We then computed the mean squared error (MSE) resulting from the above differences over all wavelet coefficients at a given scale over $64$ sample images. The resulting histograms for each perturbed layer on Fig.~\ref{fig:level_scale10000} shows that the mismatch is stronger at scales corresponding the depth of the distorted layer. The same analysis is provided at 50000 training iteration of the VAE training on Fig.~\ref{fig:level_scale}, and 50000 training iterations including SDR optimization on Fig.~\ref{fig:level_sdr_scale}.
\subsubsection{Extrapolation of specific features}
To justify that extrapolation as introduced in Sec.~\ref{sec:metaclass} and illustrated in Example~\ref{expl:eyegen} is relevant in the context of deep generative models, we show now apply stretching to specific visual features. For that we rely on the approach of \citet{besserve2018counterfactuals} to identify modules of channels in hidden layers that encode specific properties of the output images in a disentangled way. We apply this procedure on the VAE described above, and identify a group of channels distributed across hidden layers encoding eyes. We then applied the horizontal stretch described in previous sections, but only to activations of the channels in the intermediate layer that belong to the module encoding properties of the eyes. The resulting counterfactual samples, shown on Suppl. Fig.~\ref{fig:extrapolFace} (top panel), exhibit faces with disproportionate eyes, in the vein of the deformations that illustrators often apply to fictional characters that can be observed for examples in cartoons or animation movies.

\newpage

\section{Supplemental figures}

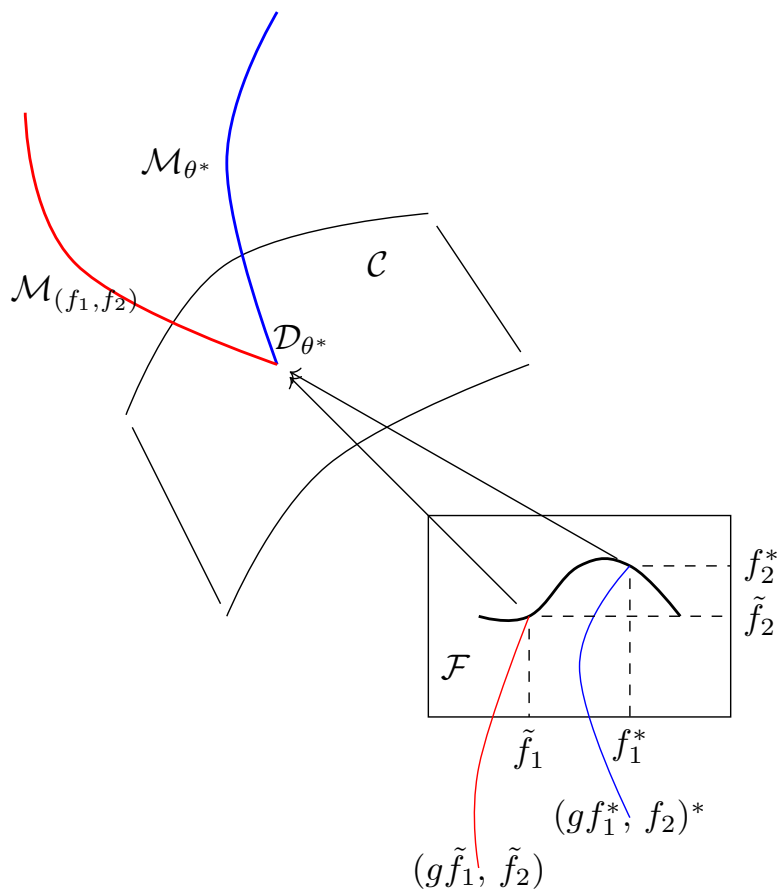
\begin{figure}[h]
	\centering
	\begin{tikzpicture}
\node (v1) at (-2.5,-1) {};
\node (v2) at (-3.5,1) {};
\node (v4) at (-0.5,3) {};
\node (v3) at (0.5,1.5) {};
\draw  (v1) edge (v2);
\draw  (v3) edge (v4);
\draw  plot[smooth, tension=.7] coordinates {(v2) (-2.5,2.5) (v4)};
\draw  plot[smooth, tension=.7] coordinates {(v1) (-1.5,0.5) (v3)};
\node  (v5) at (-2,1.5) {};
\node  (v5lab) at (-1.75,1.75) {$\mathcal{D}_{\theta^*}$};

\draw [thick,blue] plot[smooth, tension=.7] coordinates {(v5) (-2.5,3.5) (-2,5)};
\draw [thick,red] plot[smooth, tension=.7] coordinates {(v5) (-4,2.5) (-4.5,4)};
\node (v6) at (-0.5,0) {};
\node at (-0.5,-2) {};
\node (v7) at (2.5,-2) {};
\node at (2.5,0) {};
\draw  (v6) rectangle (v6);
\draw  (v7) rectangle (v7);
\draw  (v6) rectangle (2.5,-2);
\node (v8) at (0.5,-1) {};
\node (v9) at (1.5,-0.5) {};
\draw  plot[smooth, tension=.7] coordinates {(v8)};
\draw [thick] plot[smooth, tension=.7] coordinates {(0,-1) (v8) (1,-0.5) (v9) (2,-1)};
\draw [red] plot[smooth, tension=.7] coordinates {(v8) (0,-2.5) (0,-3.5)};
\draw [blue] plot[smooth, tension=.7] coordinates {(v9) (1,-1.5) (1.5,-3)};
\node [left] at (-2.5,3.5) {$\mathcal{M}_{\theta^*}$};
\node [below] at (-4,2.5) {$\mathcal{M}_{(f_1,f_2)}$};
\node at (-1,2.5) {$\mathcal{C}$};
\node at (0,-3.5) {$(g\tilde{f}_1,\,\tilde{f}_2)$};
\node at (1.5,-3) {$(g{f}_1^*,\,{f}_2)^*$};
\node [below] (v13) at (0.5,-2) {$\tilde{f}_1$};
\node [right] (v12) at (2.5,-1) {$\tilde{f}_2$};
\node [below] (v10) at (1.5,-2) {$f_1^*$};
\node [right] (v11) at (2.5,-0.5) {${f}_2^*$};
\draw [dashed] (v9) edge (v10);
\draw [dashed] (v9) edge (v11);
\draw [dashed] (v8) edge (v12);
\draw [dashed] (v8) edge (v13);
\draw  [->](v9) edge (v5);
\draw  [->](v8) edge (v5);
\node [right] at (-0.5,-1.5) {$\mathcal{F}$};
\end{tikzpicture}
	\caption{Illustration of meta-class (see Section~\ref{sec:metaclass}), and lack of $\mathcal{G}$-equivalence: although the two solutions in $\mathcal{F}$ both generate the true distribution in $\mathcal{C}$, applying group transformations leads to different meta-classes. \label{fig:manifold}}
\end{figure}


\begin{figure*}[h]
	\centering
	\begin{subfigure}[t]{0.19\linewidth}
		\includegraphics[scale=0.169]{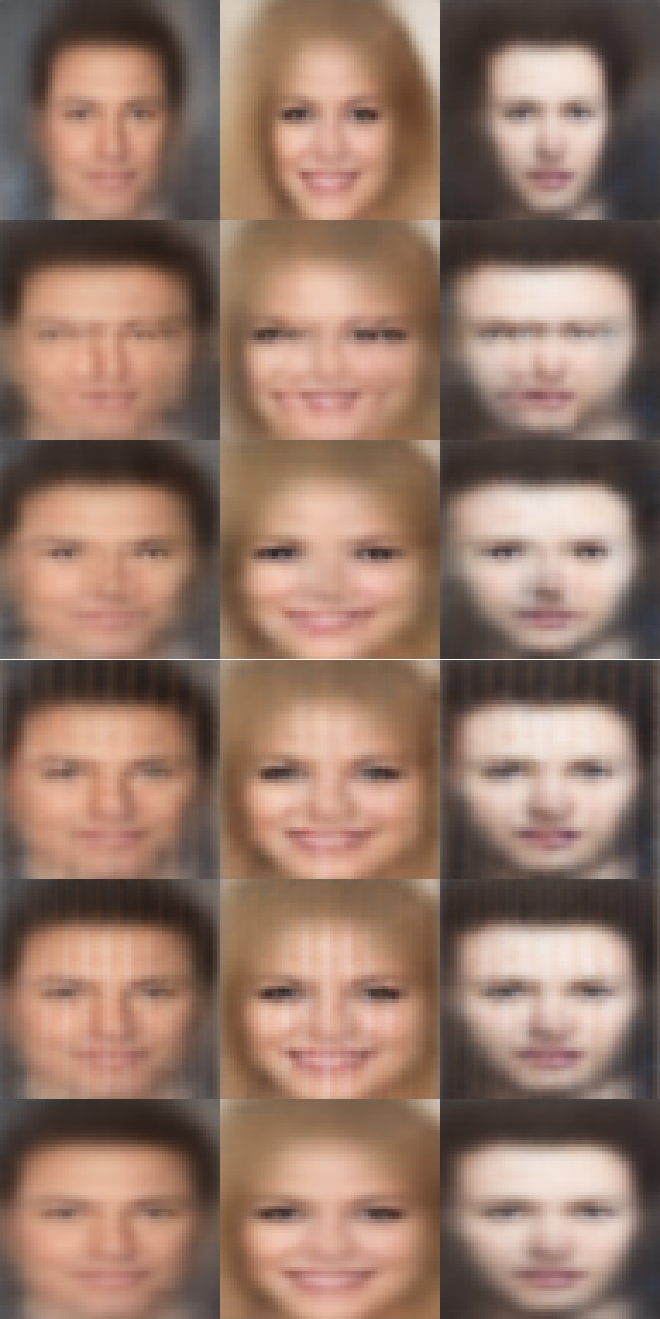}
		\caption{\label{fig:level_def10000}}
	\end{subfigure}
	\hfill
	\begin{subfigure}[t]{0.19\linewidth}
		\includegraphics[scale=0.08]{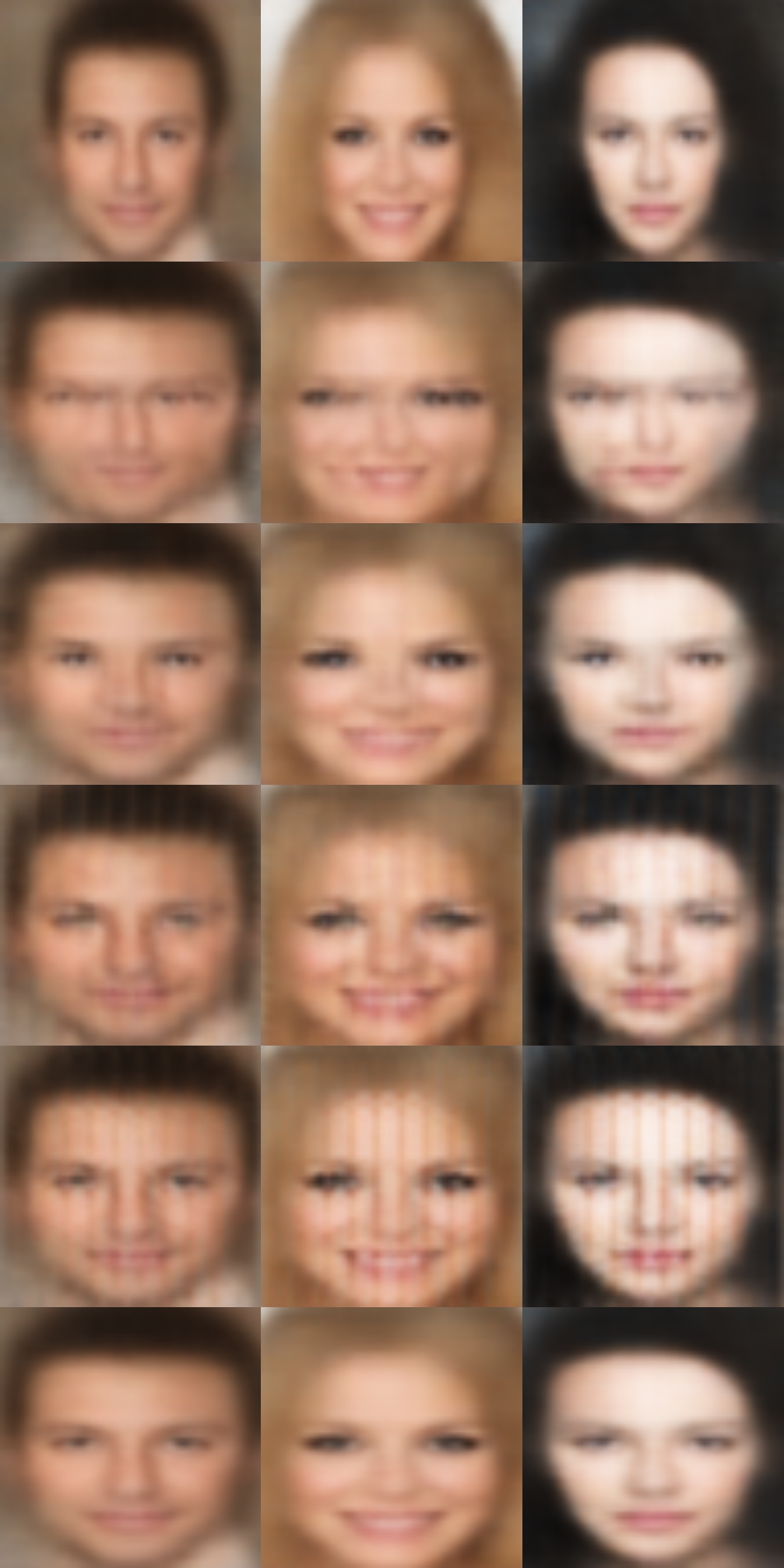}
		\caption{\label{fig:level_def}}
	\end{subfigure}
	\hfill
	\begin{subfigure}[t]{0.19\linewidth}
		\includegraphics[scale=0.08]{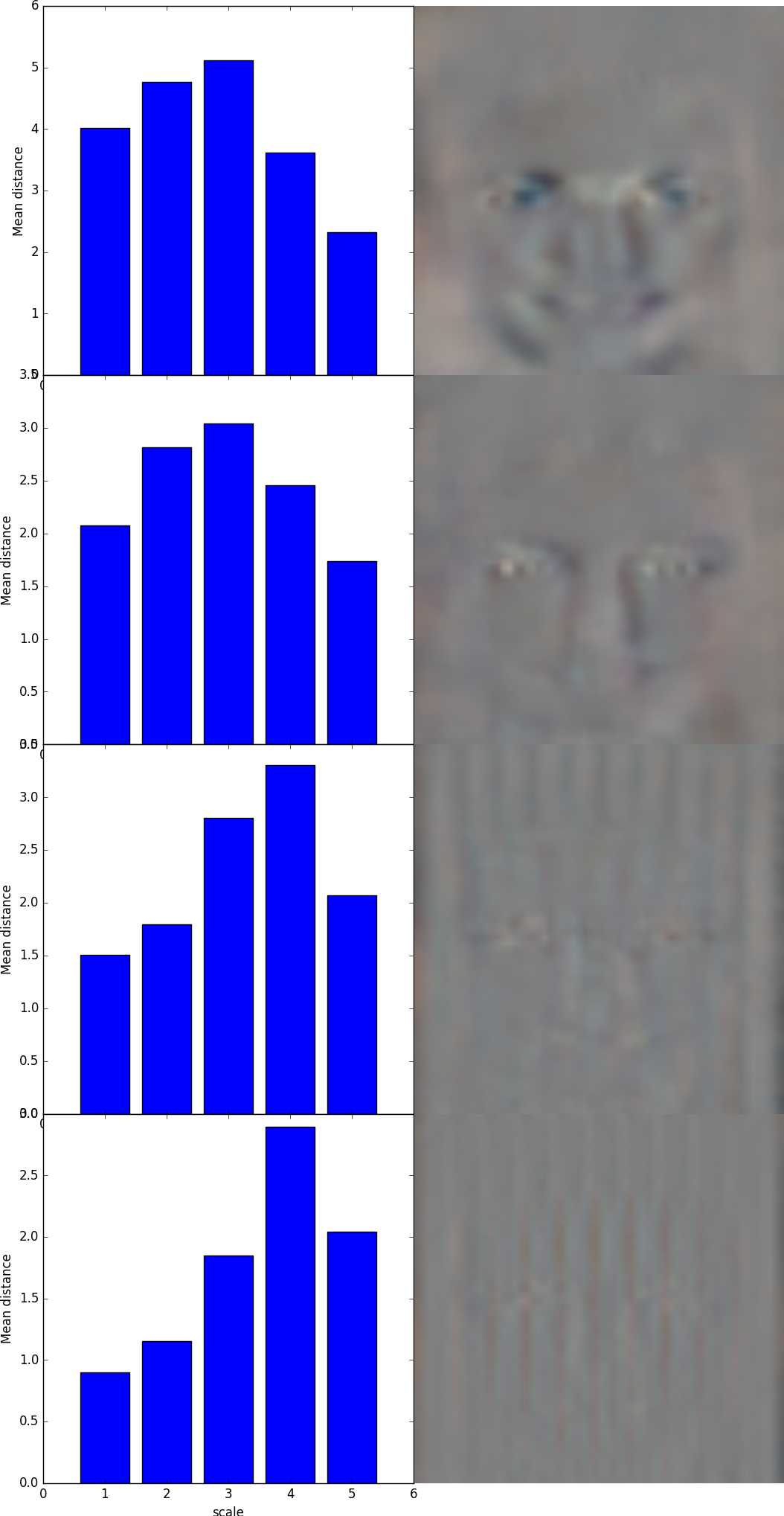}
		\caption{\label{fig:level_scale10000}}
	\end{subfigure}
	\hfill
	\begin{subfigure}[t]{0.19\linewidth}
		\includegraphics[scale=0.112]{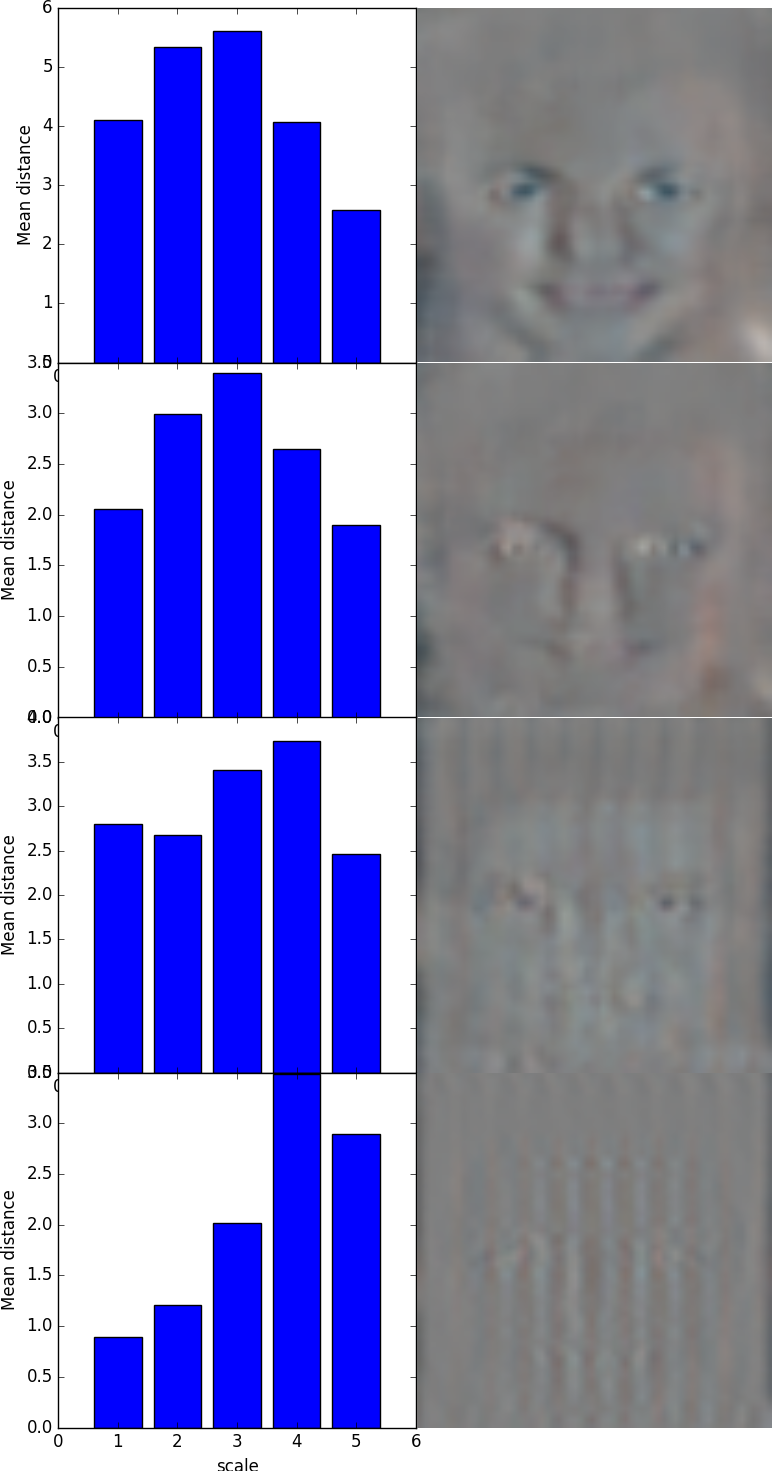}
		\caption{\label{fig:level_scale}}
	\end{subfigure}
	\hfill  \begin{subfigure}[t]{0.19\linewidth}
		\includegraphics[scale=0.08]{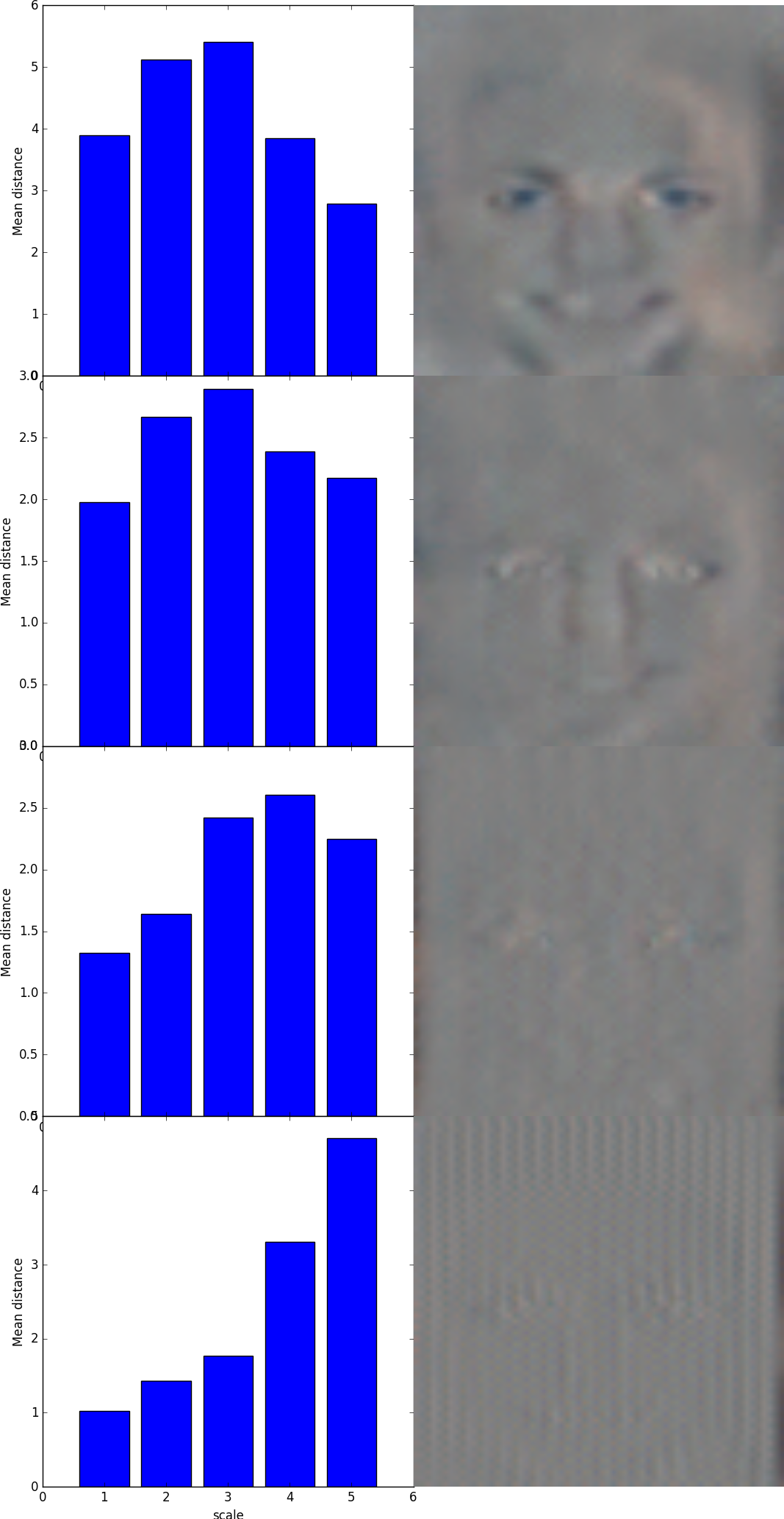}
		\caption{\label{fig:level_sdr_scale}}
	\end{subfigure}
	\caption{VAE distortion experiment. (a) From top to bottom: normal image generated by the VAE after 10000 training iterations, image resulting from distorting
		coarse/intermediate/fine/image level layer, stretched original. (b) Same as a, for 50000 training iterations. (c) Quantitative analysis for 10000 iterations. Left: Distortion levels at different wavelet scales (1:coarsest, 5:finest).
		Right: Difference with original stretched. From top to bottom: perturbation on coarse, intermediate, fine and image level, respectively. (d) Same as c for 50000 iterations. (e) Same as c but \textbf{with SDR optimization}. \label{fig:extVAEpert}}
\end{figure*}

%

\begin{figure}
	\centering
	\includegraphics[width=\linewidth]{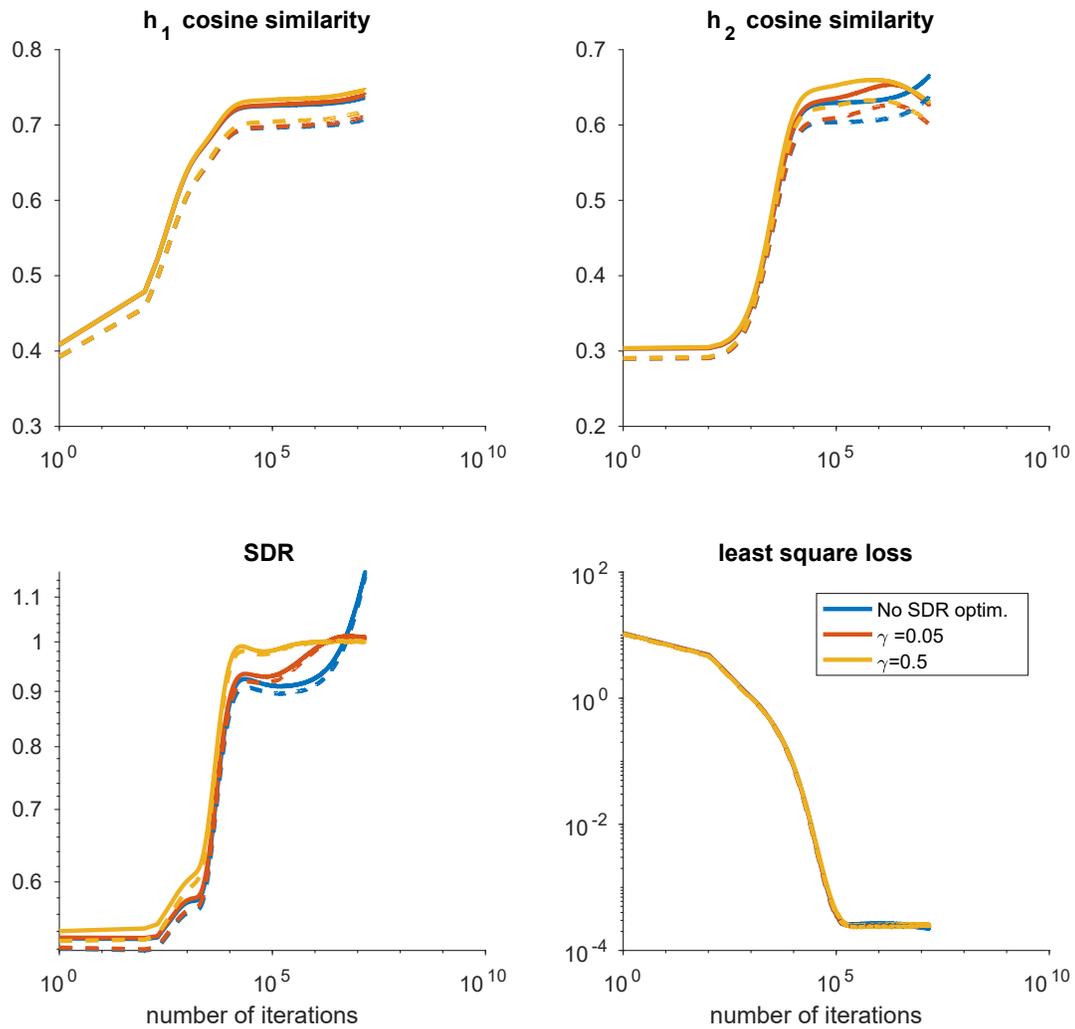}
	\caption{Evolution of the optimization of Model~\ref{model:LTLCNN} without SDR regularization (blue), and with SDR regularization (magenta and yellow, with different regularization weights $\gamma$). To represent the variability of the results for all plots and optimization types, the solid plot represents the mean plus standard error and the dashed plot represents mean minus standard error. \label{fig:optimToy}}
\end{figure}

\begin{figure}[h]
	\centering
	\hspace*{-1.8cm}\includegraphics[width=1.2\columnwidth,height=2.5cm]{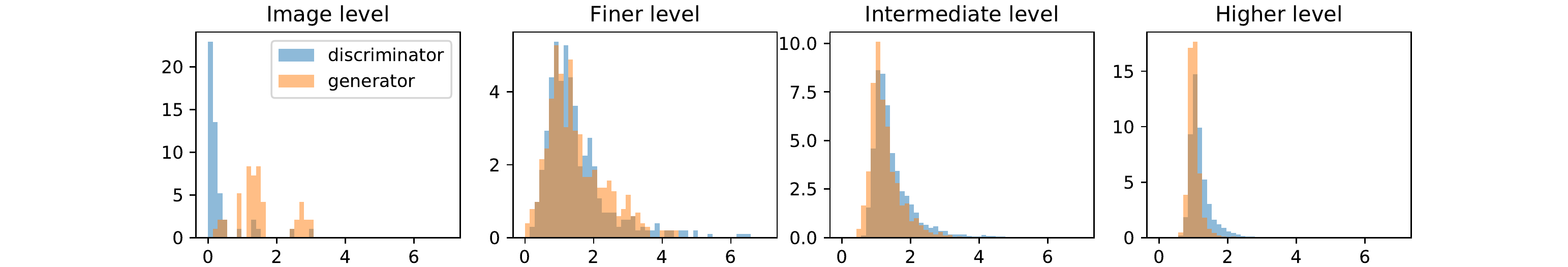}
	\caption{Superimposed SDR histograms of trained GAN generator and discriminator.\label{fig:histSICGANtrainingComp}}
\end{figure}

\begin{figure}[h]
	\centering
	\includegraphics[width=.75\linewidth]{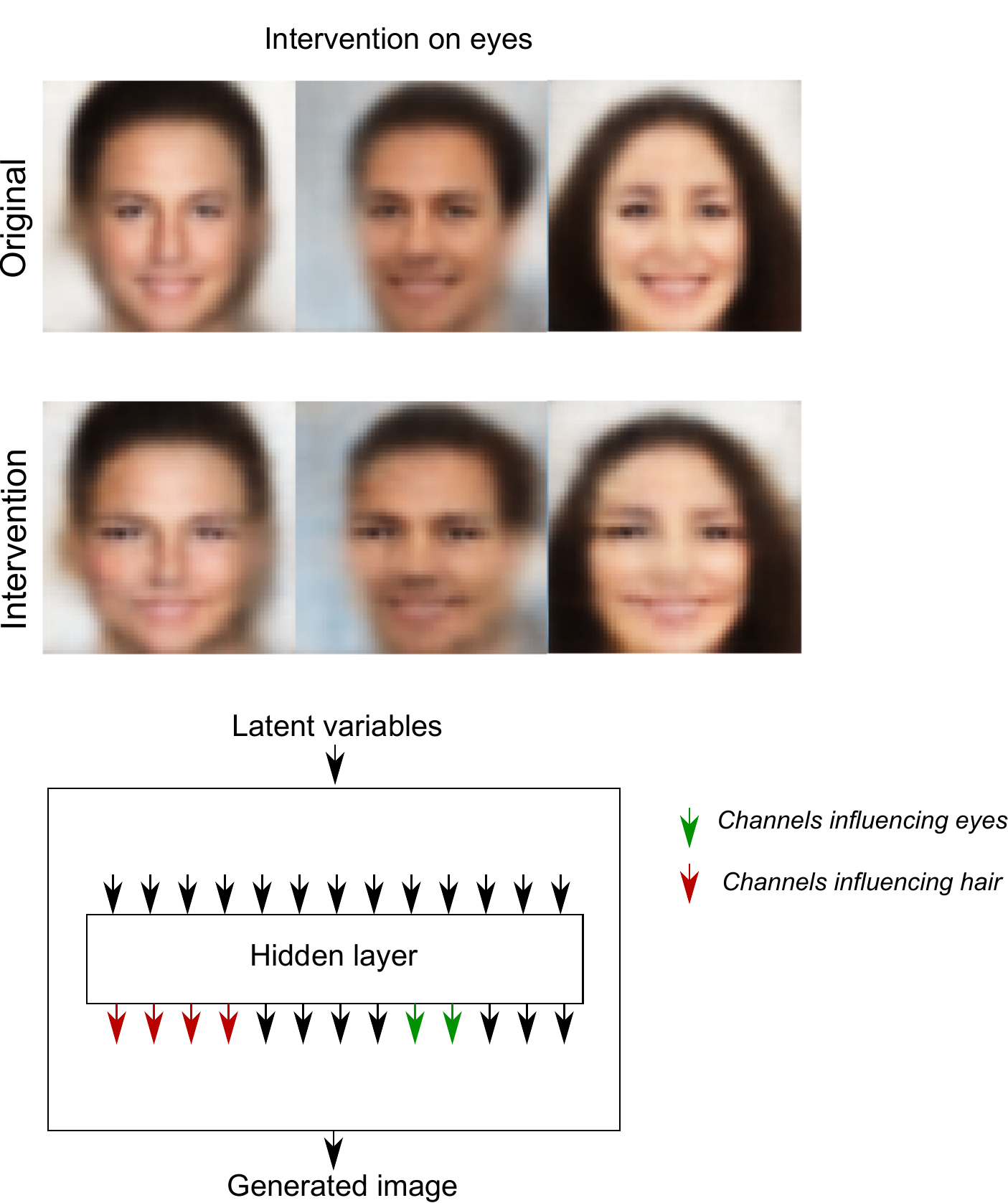}
	\caption{Targeted extrapolations. Top panel: Example extrapolated VAE samples generated by applying two different manipulations to hidden layers. Top panel: stretching of the intermediate layer channels' activations encoding the eyes. Bottom panel: illustration of the channels identified by the method of \citet{besserve2018counterfactuals} in hidden layers, used to perform targeted extrapolations. \label{fig:extrapolFace}}
\end{figure}

\end{document}